\documentclass[journal,transmag]{IEEEtran}
\usepackage{soul}
\usepackage{url}
\usepackage[hidelinks]{hyperref}
\usepackage[utf8]{inputenc}
\usepackage[small]{caption}
\usepackage{graphicx}
\usepackage{epstopdf}
\usepackage{amsmath}
\usepackage{amssymb}
\usepackage{amsthm}
\usepackage{booktabs}
\usepackage[noend]{algpseudocode}
\usepackage{amsfonts}
\usepackage{bm}
\usepackage{subfigure}
\usepackage{cite}
\usepackage{multirow}
\usepackage{longtable}
\usepackage[ruled]{algorithm2e}
\usepackage{color}
\usepackage{arydshln}

\newtheorem{theorem}{Theorem}

\begin{document}
\title{Universal Multi-Source Domain Adaptation}

\author{\IEEEauthorblockN{Yueming Yin, 
Zhen Yang,~\IEEEmembership{Senior Member,~IEEE}, 
Haifeng Hu, and 
Xiaofu Wu}
\thanks{The authors are with Nanjing University of Posts and Telecommunications, Nanjing 210003, China.

Zhen Yang is with Key Lab of Broadband Wireless Communication and Sensor Network Technology, Ministry of Education, Nanjing University of Posts and Telecommunications, Nanjing 210003, China.

Haifeng Hu and Xiaofu Wu are with National Engineering Research Center of Communications and Networking, Nanjing University of Posts and Telecommunications, Nanjing 210003, China.

Corresponding author: Yueming Yin (email: 1018010514@njupt.edu.cn) and Zhen Yang (email: yangz@njupt.edu.cn).}}

\maketitle

\begin{abstract}
Unsupervised domain adaptation enables intelligent models to transfer knowledge from a labeled source domain to a similar but unlabeled target domain. Recent study reveals that knowledge can be transferred from one source domain to another unknown target domain, called Universal Domain Adaptation (UDA). However, in the real-world application, there are often more than one source domain to be exploited for domain adaptation. In this paper, we formally propose a more general domain adaptation setting, universal multi-source domain adaptation (UMDA), where the label sets of multiple source domains can be different and the label set of target domain is completely unknown. The main challenges in UMDA are to identify the common label set between each source domain and target domain, and to keep the model scalable as the number of source domains increases. To address these challenges, we propose a universal multi-source adaptation network (UMAN) to solve the domain adaptation problem without increasing the complexity of the model in various UMDA settings. In UMAN, we estimate the reliability of each known class in the common label set via the prediction margin, which helps adversarial training to better align the distributions of multiple source domains and target domain in the common label set. Moreover, the theoretical guarantee for UMAN is also provided. Massive experimental results show that existing UDA and multi-source DA (MDA) methods cannot be directly applied to UMDA and the proposed UMAN achieves the state-of-the-art performance in various UMDA settings.
\end{abstract}


\section{Introduction}
\label{Introduction}
\IEEEPARstart{I}{n} the past years, the development of deep learning greatly promotes the research on supervised image classification \cite{qi2008two,joshi2012scalable,mensink2013distance,akata2013good,hayat2014deep,akata2015label,rao2018runtime}. At present, many large-scale image classification datasets are publicly available, and deep learning models performs very well in many scenarios. However, large-scale image annotation is time-consuming and even prohibitive. This requires that deep learning models trained on the labeled dataset can generalize to unseen but similar unlabeled data. If unlabeled data and labeled data have different characteristics independent of label, they have been sampled from two different domains \cite{busto2018open}. Therefore, domain adaptation (DA) aims to mitigate the impact of the domain gap during the process of knowledge transfer. A very challenging setting is unsupervised DA, where the source domain is fully labeled and the target domain is unlabeled. This kind of domain gap brings great challenge for the DA methods to transfer knowledge from source domain to target domain.

\begin{figure}
\centering
\includegraphics[height=1.8in,width=0.8\linewidth]{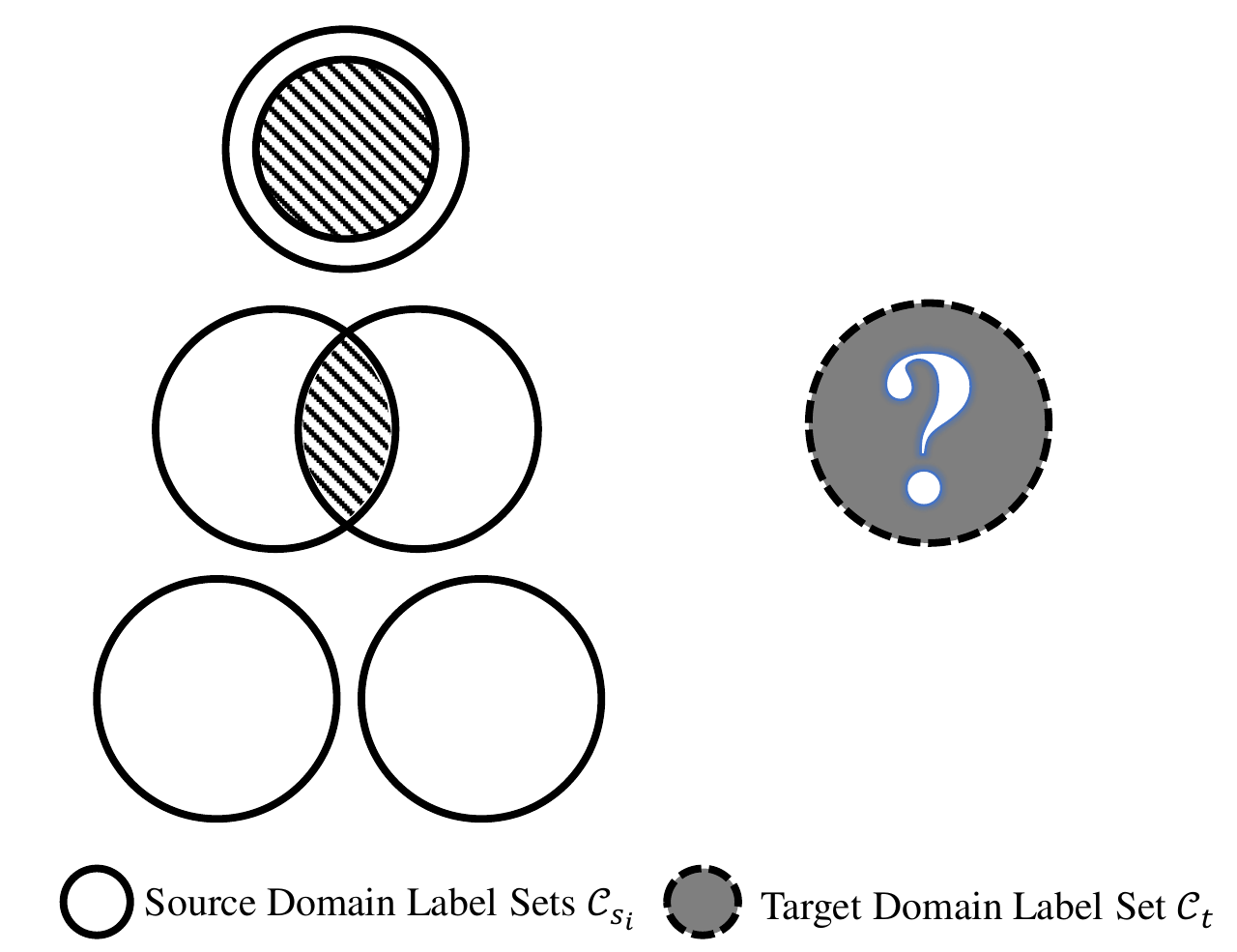}
\caption{Universal Multi-Source Domain Adaptation.}
\label{UMDA}
\end{figure}

Depending on whether the label set of the target domain is known, existing unsupervised DA methods fall into two types: non-universal DA and universal DA \cite{you2019universal,abs-2001-05071, abs-2002-07953, kundu2020universal}. And non-universal DA methods generally include closed set DA \cite{tzeng2014deep,long2015learning,ganin2016domain,haeusser2017associative,tzeng2015simultaneous,long2016unsupervised,DBLP:conf/cvpr/TzengHSD17,saito2018maximum,long2018conditional,Ben-DavidBCKPV10,bousmalis2017unsupervised,sankaranarayanan2018generate,hu2018duplex,liu2018detach,murez2018image,huang2018auggan,DBLP:conf/cvpr/VolpiMSM18,chadha2019improved,chen2020domain,chen2019graph,li2020generating,li2020aligning}, partial DA \cite{DBLP:conf/cvpr/CaoL0J18,zhang2018importance,cao2018partialeccv} and open set DA \cite{panareda2017open, saito2018open,luo2020progressive}. These non-universal DA settings suppose that all label sets of source and target domains are known in advance, which is relatively simple to deal with. Recently, the universal DA has been receiving more attention due to the fact that the label set of target domain is completely unknown. However, existing universal DA methods \cite{you2019universal,abs-2001-05071, abs-2002-07953, kundu2020universal} focus on single-source scenarios and cannot be directly applied to multi-source scenarios. For example, we have a source domain with label set A and another source domain with label set B, and these two label sets A and B are different. If we want to train a classifier to identify samples of the target domain belonging to the label set A and B, the better way is to pre-train the classifier on both of these two source domains. However, in the combined dataset, domain gap exists within the multiple distinct source domains. Unlike single-source DA \cite{Ben-DavidBCKPV10}, the same source distribution assumption doesn't hold in the case of multiple source domains. To this end, we propose a more general setting of DA as shown in Fig. \ref{UMDA}, called universal multi-source domain adaptation (UMDA). In this setting, multiple source domains are available and the label set of target domain is completely unknown. The label set of source domains is known but the relationship between any two source label sets might be contained, intersected, or even disjoint. Therefore, more knowledge is expected to be transferred from multiple source domains to the target domain, and the classifier trained on multiple source domains can be capable of identifying more classes in the target domain.

Differing from universal single-source domain adaptation \cite{you2019universal} (USDA), UMDA poses two new challenges. First, in order to identify a known class in the target domain, we have to consider the domain discrepancy not only between each source domain and target domain, but also between any two distinct source domains. Second, as the number of source domains increases, the relationship between label sets of source and target domains becomes more complex. Some multi-source adaptation methods \cite{xu2018deep, zhao2018adversarial, ZhaoWZGLS0HCK20,zhao2020madan,li2020multi,lin2020multi,li2020mutual,PengBXHSW19,0002ZWCMG18} have been proposed. However, these methods cannot provide the promising performance in the target domain. Furthermore, with the increase of the number of source domains, the model complexity of these methods makes it difficult to implement in an efficient way.

To tackle these two challenges, we design a Universal Multi-Source Adaptation Network (UMAN) with a novel weighting mechanism to identify the common classes of source and target domains, which mitigates the domain gaps between a large number of source domains and an unknown target domain without increasing the complexity of the model. Specifically, we first introduce the \emph{margin vector} to UMDA, which defined as the empirical prediction margin with respect to each pseudo-label as shown in Fig. \ref{motivation}. Prediction margin measures the confidence in assigning a target sample to its pseudo-label. Therefore, a pseudo-label with a larger prediction margin is more likely to be a common class between source and target domains. It can be seen that we can accurately identify the common label set by using margin vector related operations. In UMAN, we also design a joint cross-entropy loss and a joint adversarial loss to align not only the distributions between the source domains but also between the source and target domains in their common label set. Note that in the joint adversarial loss, each class in different source domains is weighted by the same \emph {margin vector}, so the complexity of UMAN become insensitive to the number of source domains. 

Moreover, we provide theoretical analysis of the two joint losses in UMAN, and obtain two theoretical guarantees. First, optimizing the joint cross-entropy loss is equivalent to aligning the distribution of source domains in their common label set. Second, optimizing the joint adversarial loss is equivalent to aligning the distribution between the center of multiple source domains and target domain in their common label set.

The main contributions of this paper can be summarized as:
\begin{enumerate}
\item We propose a more practical Universal Multi-Source Domain Adaptation (UMDA) setting, where the label sets of multiple source domains can be different and the label set of target domain is completely unknown. Thus, the knowledge of multiple source domains can be transferred to target domain in an unsupervised way.
\item We propose a Universal Multi-Source Adaptation Network (UMAN) with a novel margin vector for UMDA. The common label set can be identified via the margin vector, which helps adversarial training to better align the distributions of multiple source domains and target domain in an efficient way, and does not incur extra model complexity in various UMDA settings. Moreover, we make theoretical analysis of the two losses of UMAN, which provides the theoretical guarantees for alignment of distribution of source and target domains in the common label set.
\item Massive experimental results show that the proposed UMAN works stably across various UMDA settings and outperforms the state-of-the-art methods by large margins.

\end{enumerate}
\begin{figure}
\centering
\includegraphics[height=1.8in,width=1\linewidth]{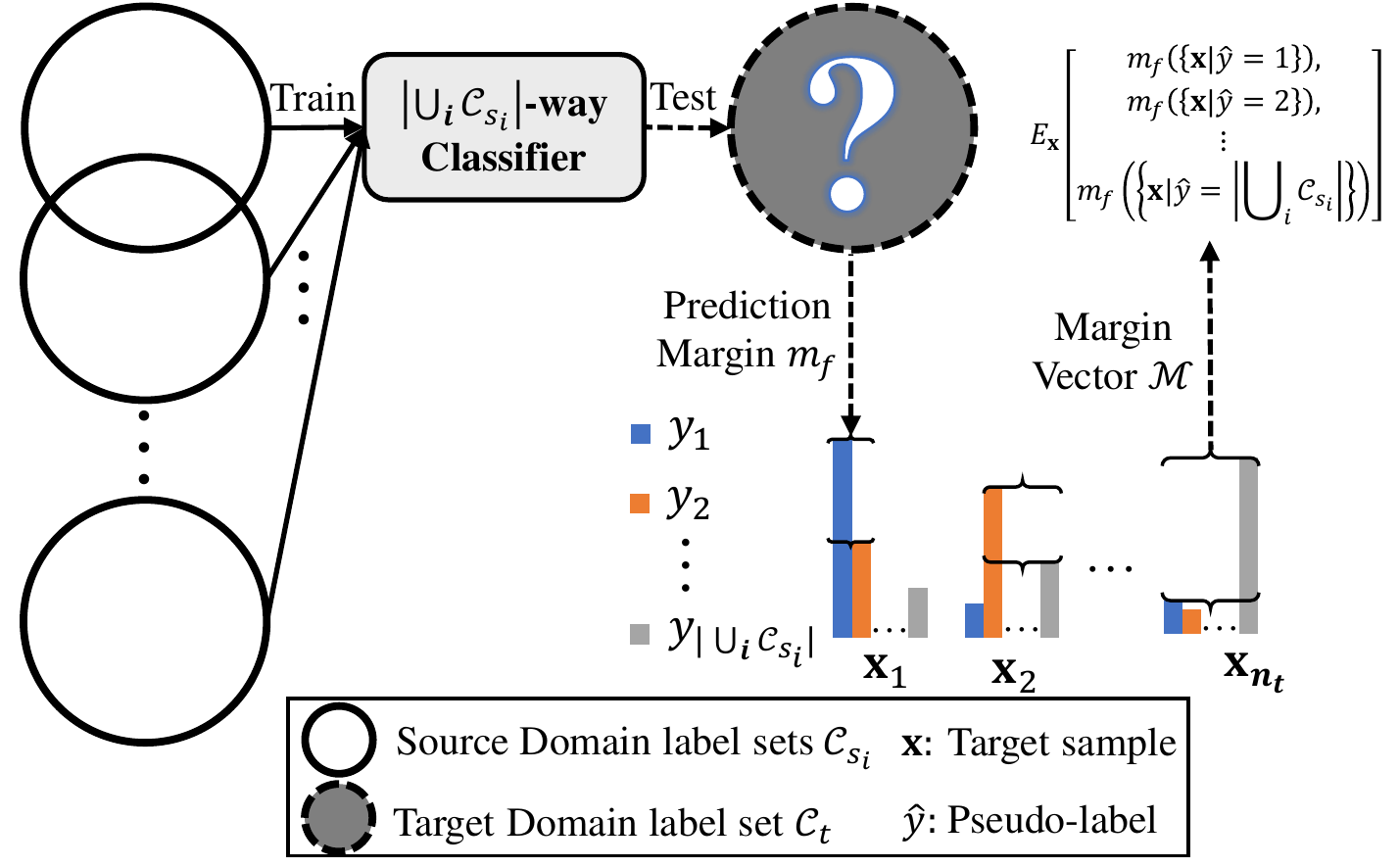}
\caption{An illustration of margin vector. We employ the prediction margin as the reliability of pseudo-label in the common label set.}
\label{motivation}
\end{figure}

\section{Related Work}
We briefly review recent domain adaptation methods in this section. According to the number of source domains, these methods fall into two categories: single-source domain adaptation and multi-source domain adaptation. In single-source domain adaptation, according to the constraint on the relationship between source and target label sets, methods are divided into closed set domain adaptation, partial domain adaptation, open set domain adaptation and universal domain adaptation. In multi-source domain adaptation, existing methods primarily focus on closed set setting.

\subsection{Single-Source Domain Adaptation}
\textbf{Closed Set Single-Source Domain Adaptation.} Closed set single-source domain adaptation focuses on mitigating the impact of the domain gap between source and target domains. Closed set single-source domain adaptation methods mainly contain two categories: feature adaptation and generative model. Feature adaptation methods \cite{tzeng2014deep,long2015learning,ganin2016domain,haeusser2017associative,tzeng2015simultaneous,long2016unsupervised,DBLP:conf/cvpr/TzengHSD17,saito2018maximum,long2018conditional,chen2020domain,li2020aligning} diminish the feature distribution discrepancy between source and target domains by minimizing statistical distances, such as $\mathcal{H}$-divergence \cite{Ben-DavidBCKPV10}, $\mathcal{H} \Delta \mathcal{H}$-divergence \cite{BlitzerCKPW07} and MMD \cite{long2015learning,chen2019graph}. Methods based on generative models \cite{bousmalis2017unsupervised,sankaranarayanan2018generate,hu2018duplex,liu2018detach,murez2018image,huang2018auggan,DBLP:conf/cvpr/VolpiMSM18,chadha2019improved,li2020generating} synthesize target samples as data augmentation and match domains in pixel or feature levels.

\textbf{Partial Single-Source Domain Adaptation.} Partial single-source domain adaptation (PSDA) \cite{DBLP:conf/cvpr/CaoL0J18,zhang2018importance,cao2018partialeccv} transfers a learner from a big source domain to a small target domain, where the source label set contains the target label set. To solve PSDA, Cao \textit{et al}. \cite{DBLP:conf/cvpr/CaoL0J18} utilized multiple domain discriminators with class-level and instance-level weighting mechanism to perform per-class adversarial distribution matching. Zhang \textit{et al}. \cite{zhang2018importance} constructed an auxiliary domain discriminator to estimate the probability of a source sample being similar to the target domain. Cao \textit{et al}. \cite{cao2018partialeccv} further improved PSDA by employing an adversarial network and jointly applying class-level weighting on the source classifier.

\textbf{Open Set Single-Source Domain Adaptation.} Busto \textit{et al}. \cite{panareda2017open} first proposed open set single-source domain adaptation (OSSDA). The classes private to both domains are recognized as an “unknown” class. They use an Assign-and-Transform-Iteratively (ATI) algorithm to map target samples to source classes and then train SVMs for final classification. Saito \textit{et al}. \cite{saito2018open} modified OSSDA by requiring no data of the source private label set and extends the source classifier by adding an “unknown” class and trains it adversarially among classes. Recently, Luo \textit{et al}. \cite{luo2020progressive} addressed conditional shift by constructing graph neural networks for OSSDA.

\textbf{Universal Single-Source Domain Adaptation.} For these aforementioned single-source DA methods, prior knowledge about label sets is assumed to be known. You \textit{et al}. \cite{you2019universal} first proposed universal single-source domain adaptation (USDA), which requires no prior knowledge on the label sets. In \cite{you2019universal}, a novel criterion was proposed to quantify the transferability of each sample and weight samples during domain adversarial training. Recently, Lifshitz \textit{et al}. \cite{abs-2001-05071} adopted a sample selection approach to promote the performance for USDA. Saito \textit{et al}. \cite{abs-2002-07953} solved USDA problem through self-supervision. Kundu \textit{et al}. \cite{kundu2020universal} proposed a source-free USDA method for real-time adaptation.

\subsection{Closed Set Multi-Source Domain Adaptation}
Existing methods for multi-source domain adaptation primarily focus on a relatively simple application scenario. That is, the target data is unlabeled but available during the training process, the source data is fully labeled, the source and target data are observed in the same data space, and the label sets of all sources and the target are the same \cite{abs-2002-12169}. Here, we call these methods as closed set multi-source domain adaptation (CSMDA). The idea behind CSMDA is based on the seminal theoretical model \cite{BlitzerCKPW07,Ben-DavidBCKPV10}. Mansour \textit{et al}. \cite{MansourMR08} assumed that the target distribution can be approximated by a mixture of the multiple source distributions. Therefore, weighted combination of source classifiers has been widely employed for CSMDA. Moreover, tighter cross domain generalization bound and more accurate measurements on domain discrepancy can provide intuitions to derive effective CSMDA algorithms. Hoffman \textit{et al}. \cite{HoffmanMZ18} derived a novel bound using DC-programming and calculated more accurate combination weights. Zhao \textit{et al}. \cite{zhao2018adversarial} extended the generalization bound of seminal theoretical model to multiple sources under both classification and regression settings. Besides the domain discrepancy between the target and each source \cite{HoffmanMZ18,zhao2018adversarial}, Li \textit{et al}. \cite{LiMDC18} also considered the relationship between pairwise sources and derived a tighter bound on weighted multi-source discrepancy.

In closed set multi-source domain adaptation, category shift was first formally considered in \cite{xu2018deep}, where each source domain may not contain all the classes in the target domain. However, the union of all source label sets is the same as target label set, which can be still categorized as closed set multi-source domain adaptation.

\section{Our Approach}
In this section, we introduce the scenario of Universal Multi-Source Domain Adaptation (UMDA) and then propose a novel Universal Multi-Source Adaptation Network (UMAN) for this scenario.
\subsection{Problem Setting}
\label{Problem Setting}
In Universal Multi-Source Domain Adaptation, $M$ source domains $\mathcal{D}_{s_i}=\left\{\left(\mathbf{x}_{j}^{s_i}, y_{j}^{s_i}\right)\right\},i=1,2,\cdots,M$ consisting of $n_{s_i}$ labeled samples and a target domain $\mathcal{D}_{t}=\left\{\left(\mathbf{x}_{k}^{t}\right)\right\}$ of $n_t$ unlabeled samples are provided at training. Note that the source data are sampled from distribution $p_i$ while the target data from distribution $q$. We use $\mathcal{C}_{s_i}$ to denote the label set of source domain $\mathcal{D}_{s_i}$ and $\mathcal{C}_t$ the label set of target domain. $\mathcal{C}_i=\mathcal{C}_{s_i} \cap \mathcal{C}_{t}$ is the common label set shared by $\mathcal{D}_{s_i}$ and $\mathcal{D}_{t}$. $\overline{\mathcal{C}}_{s_i}=\mathcal{C}_{s_i} \setminus \mathcal{C}_i$ and $\overline{\mathcal{C}}_{t}=\mathcal{C}_{t} \setminus \{\bigcup_{i} \mathcal{C}_i\}$ represent the label set private to the source domain $\mathcal{D}_{s_i}$ and the target domain $\mathcal{D}_{t}$ respectively. And $\mathcal{C}_s=\bigcup_{i}\mathcal{C}_{s_i}$ is the label set which contains all the known labels in source domains, $\mathcal{C}=\bigcup_{i}\mathcal{C}_i$ is the union of common label sets between each source domain and the target domain, and $\overline{\mathcal{C}}_s=\bigcup_{i}\overline{\mathcal{C}}_{s_i}$ is the common label set between source domains that does not appear in target domain. $p_{\mathcal{C}_i}$ and $p_{\overline{\mathcal{C}}_{s_i}}$ are used to denote the distribution of source data with labels in the label set $\mathcal{C}_i$ and $\overline{\mathcal{C}}_{s_i}$ respectively, and $q_{\mathcal{C}}$, $q_{\mathcal{C}_i}$, $q_{\overline{\mathcal{C}}_t}$ for target distributions with labels in the label set $\mathcal{C}$, $\mathcal{C}_i$, $\overline{\mathcal{C}}_t$ respectively. Note that the target data are fully unlabeled, and the target label sets (inaccessible at training) are only used for defining the UMDA problem.

The definition of commonness between two domains introduced in \cite{you2019universal} is the Jaccard distance of two label sets: $\xi=\frac{\left|\mathcal{C}_{s} \cap \mathcal{C}_{t}\right|}{\left|\mathcal{C}_{s} \cup \mathcal{C}_{t}\right|}$. Closed set domain adaptation is a special case of universal domain adaptation when $\xi=1$. In UMDA, Jaccard distance is modified as $\xi_i=\frac{\left|\mathcal{C}_{s_i} \cap \mathcal{C}_{t}\right|}{\left|\mathcal{C}_{s_i} \cup \mathcal{C}_{t}\right|}$ associated with the $i$-th source domain and the target domain. Similarly, $\xi_{ij}=\frac{\left|\mathcal{C}_{s_i} \cap \mathcal{C}_{s_j}\right|}{\left|\mathcal{C}_{s_i} \cup \mathcal{C}_{s_j}\right|}$ is defined as Jaccard distance between source label sets. Both domain gaps and category gaps exist in UMDA setting. In other words, domain gaps exist not only between each source domain and target domain but also between any two source domains, i.e. $p_{\mathcal{C}_i} \neq q_{\mathcal{C}_i}$ and $p_i \neq p_j$. The main task for UDA is to mitigate the impact of $p_{\mathcal{C}_i} \neq q_{\mathcal{C}_i}$. Meanwhile, the learning model should distinguish between target samples coming from known classes in $\mathcal{C}$ and unknown classes in $\overline{C}_{t}$. Finally, the model should be learned to minimize the target risk in $\mathcal{C}=\bigcup_{i}\mathcal{C}_i$, i.e. $\min \mathbb{E}_{(\mathbf{x}, y) \sim q_{\mathcal{C}}}[f(\mathbf{x}) \neq y]$.

\subsection{Margin Vector}
We first introduce a class-wise weighting mechanism guaranteed by margin theory to UMDA. The margin between features and the classification surface makes an important impact on designing generalizable classifier. Therefore, a margin theory for classification was developed by \cite{koltchinskii2002empirical} and \cite{zhang2019bridging}, where the 0-1 loss for classification is replaced by the marginal loss. Differ from \cite{koltchinskii2002empirical} and \cite{zhang2019bridging}, the prediction margin in this paper is to quantify the probability of a target label belonging to the common label set. In particular, we define the \emph{margin} of a hypothesis predictor (a $|\mathcal{C}_s|$-way classifier) $f$ at a pseudo-labeled example $\mathbf{x}$ of target domain as
\begin{equation}
\begin{aligned}
m_{f}(\mathbf{x}) &\triangleq f(\mathbf{x},\hat{y})-\max _{y \neq \hat{y}} f\left(\mathbf{x}, y\right),
\label{margin}
\end{aligned}
\end{equation}
\begin{equation}
\begin{aligned}
\hat{y}&\triangleq\mathop{\arg\max} _{y\in [1,|\mathcal{C}_s|]}f\left(\mathbf{x}, y\right),
\label{pseudo}
\end{aligned}
\end{equation}where $f\left(\mathbf{x}, y\right)$ is the classification probability of $\mathbf{x}$ belonging to the $y$-th class, and $\hat{y}$ is the pseudo-label of $\mathbf{x}$. This \emph{margin} measures the confidence in assigning an example to its pseudo-label. In particular, wrong pseudo-labeled samples and samples of unknown classes will have a small margin, where the classification surface intersects here. If $m_{f}(\mathbf{x})=1$, $\mathbf{x}$ is considered most likely to come from known classes. Similarly, if $m_{f}(\mathbf{x})=0$, $\mathbf{x}$ is most likely to come from unknown classes.

Then, the empirical \emph{margin vector} over a data distribution $\mathcal{D}$ is defined as
\begin{equation}
\begin{aligned}
\mathcal{M}(\mathcal{D},f)\triangleq\mathbb{E}_{\mathbf{x}\in\mathcal{D}}\left[m_{f}(\{\mathbf{x}|\hat{y}=1\}),\cdots,m_{f}(\{\mathbf{x}|\hat{y}=|\mathcal{C}_s|\})\right]^T.
\label{mv}
\end{aligned}
\end{equation} In this definition, the $i$-th dimension of the \emph{margin vector} is the empirical \emph{margin} of samples with the $i$-th pseudo-label. In UMDA, \emph{margin vector} can be used to calculated over the classifier trained on each source domain or the classifier trained on the integrated source domains, which obtains the reliability of each source classes belonging to the common label set $\mathcal{C}_i$ or $\mathcal{C}=\bigcup_{i}\mathcal{C}_i$ respectively.

\subsection{Target Margin Register and Transferability Criterion}
We first introduce a target margin register (TMR) to UMDA. The TMR is formulated as a $\mathcal{C}_{s}$-dimensional vector $\mathbf{V}_{TMR}$, which is updated in each training step. The updated rule is defined as
\begin{equation}
\begin{aligned} 
\mathbf{V}^{t+1}_{TMR}=\frac{1}{t+1}(t\times \mathbf{V}^{t}_{TMR}+ \mathcal{M}(\mathcal{D}^{b}_{t},f)),
\label{updated rule}
\end{aligned}
\end{equation}
\begin{equation}
\begin{aligned} 
\mathbf{V}^{0}_{TMR}={\underbrace{[0,0,\cdots,0]}_{|\mathcal{C}_{s}|}}^T,
\label{init}
\end{aligned}
\end{equation}where $\mathcal{M}(\mathcal{D}^{b}_{t},f)$ is the \emph{margin vector}, defined in Eq. \ref{mv}, which outputs a $\mathcal{C}_{s}$-dimensional \emph{margin vector} over $\mathcal{D}^{b}_{t}$ and $f$, $\mathcal{D}^{b}_{t}$ denotes a batch of target samples from $\mathcal{D}_{t}$, $f$ denotes the prediction function of the classifier. Note that $t\times \mathbf{V}^{t}_{TMR}$ is the accumulated \emph{margin vector} of previous $t$ steps, $t$ is the updating step and $t<T$, where $T$ is the maximum step number for training. $\mathbf{V}^{t}_{TMR}$ denotes the TMR-vector at step $t$. In Eq. \ref{updated rule}, the first term is equal to the accumulated \emph{margin vector} in the previous steps, the second term is the \emph{margin vector} of $\mathcal{D}^{b}_{t}$ and $f$ in the current training step. The definition of TMR is to calculate the average of the \emph{margin vector} corresponding to all batches of target samples during training.

Each dimension of $\mathbf{V}_{TMR}$ represents the empirical prediction margin, which means the level of reliability that the corresponding label in the source domain belonging to the label set of the target domain. Therefore, this margin can be directly used to weight a source sample $\mathbf{x}_{j} \in \mathcal{D}_{s}$ by its labels $y_j$:
\begin{equation}
\begin{aligned}
w^{s}(y_j)=\mathbf{V}_{TMR}[y_j],
\label{wsv}
\end{aligned}
\end{equation}where $w^{s}(y_j)$ indicates the probability of the source label $y_j$ belonging to the common label set $\mathcal{C}$, and $\mathbf{V}_{TMR}[y_j]$ is the dimension of $\mathbf{V}_{TMR}$ corresponding to $y_j$. Note that in Eq. \ref{wsv}, TMR-vector, shared by all source domains, has $|\mathcal{C}_s|$ components associated with all the known labels in source domains. Each class in different source domains is weighted by the same TMR-vector, so the complexity of UMAN become insensitive to the number of source domains. 

In the same way, \emph{margin vector} is utilized to weight target samples with respect to their pseudo-labels. Considering the confidence in assigning a target sample to its pseudo-label varies between target samples, the ultimate probability of each target sample $\mathbf{x}_k$ belonging to the union of common label sets $\mathcal{C}=\bigcup_{i}\mathcal{C}_i$ is defined by both class-wise and sample-wise reliabilities:
\begin{equation}
\begin{aligned}
w^t(\mathbf{x}_k,f)=m_f(\mathbf{x}_k)\cdot\mathbf{V}_{TMR}[\hat{y}_k],
\label{wtv}
\end{aligned}
\end{equation}where $m_f(\mathbf{x}_k)$ is the \emph{margin} of a target sample $\mathbf{x}_k$ calculated by the classifier $f$ trained on the source domains in Eq. \ref{margin}. $\mathbf{V}_{TMR}[\hat{y}_k]$ is the dimension of $\mathbf{V}_{TMR}$ corresponding to $\hat{y}_k$, and $\hat{y}_k$ is the pseudo label of $\mathbf{x}_k$ defined in Eq. \ref{pseudo}.

Unlike \cite{you2019universal}, we define the weighting mechanism in terms of class rather than individual sample in the source domain, and consider both class-wise weights and sample-wise reliabilities in the target domain. \textbf{Essentially}, for all source domains, samples in the same class should share a common weight, when the probability of a class belonging to $\mathcal{C}$ can be estimated. And the weight of classes in the common label set should be consistent through all the domains. Note that $w^{s}(y_j)$ (or $w^t(\mathbf{x}_k,f)$) is further normalized as \cite{you2019universal}.

\subsection{Universal Multi-Source Adaptation Network}
The main challenge of UMDA is to mitigate the impact of $p_{\mathcal{C}_i}\neq q_{\mathcal{C}_i}$. Unlike \cite{guo2018multi,zhu2019aligning,rakshit2019unsupervised,peng2019moment,xu2018deep,DBLP:conf/aaai/ZhaoWZGLS0HCK20}, we use a shared feature extractor, a shared classifier and a shared discriminator to solve UMDA problem. This means the proposed UMAN can work on a large number of source domains without increasing complexity. To this end, let $F: \mathbb{R}^{l\times w}\rightarrow \mathbb{R}^d$ be the feature extractor, and $G: \mathbb{R}^d\rightarrow \{0,1,\cdots,|\mathcal{C}_{s}|-1\}$ be the classifier. Here, $l\times w$ and $d$ denote the dimension of input samples and feature representations respectively. Input $\mathbf{x}$ from multiple domains are forwarded into $F$ to obtain the feature $\mathbf{z}=F(\mathbf{x})$, then, $\mathbf{z}$ is L2-normalized and fed into $G$ to estimate the classification probability $G^{c}(\mathbf{z})$ of $\mathbf{x}$ over the $c$-th class in $\mathcal{C}_s$. The probabilities of source input $\mathbf{x}_i$ are reserved to calculate the joint cross-entropy loss $E_{G}$ with labels in all source domains:
\begin{equation}
\begin{aligned}
E_{G}=\frac{1}{M}\sum_{i=1}^M\mathbb{E}_{(\mathbf{x}, y) \sim p_i} L\left(y, G\left(F(\mathbf{x})\right)\right),
\label{Lc}
\end{aligned}
\end{equation}where $L(\cdot,\cdot)$ denotes the cross-entropy loss, $F\left(\mathbf{x}\right)$ is the L2-normalized feature. The optimization on Eq. \ref{Lc} is equivalent to aligning the distributions between any two source domains in their common label sets, which is guaranteed by Theorem \ref{theorem1} (see below).

Let $D: \mathbb{R}^d\rightarrow \{0,1\}$ be the binary classifier discriminating input feature $\mathbf{z}$ from the source domain or the target domain. Inspired by \cite{you2019universal}, the domain classifier $D$ can be trained by weighting $w^{s}(y_j)$ defined in Eq. \ref{wsv} and $w^t(\mathbf{x}_k,G\circ F)$ defined in Eq. \ref{wtv} as
\begin{equation}
\begin{aligned} 
E_{D}=-\frac{1}{M}\sum_{i=1}^M&\mathbb{E}_{(\mathbf{x}, y) \sim p_i} w^{s}(y) \log D\left(F(\mathbf{x})\right)
\\ -\mathbb{E}_{\mathbf{x} \sim q} w^{t}&(\mathbf{x},G\circ F) \log \left(1-D\left(F(\mathbf{x})\right)\right),
\label{Ld}
\end{aligned}
\end{equation}where $G\circ F$ is the entity of the classifier $f$ in Eq. \ref{wtv}, and $M$ is the number of source domains. Through the above adversarial learning, the center of the distributions of source domains and the distribution of target domain are aligned in $\mathcal{C}$ as precisely as possible, which is guaranteed by Theorem \ref{theorem2}. Note that the distributions between any two source domains in their common label set are aligned according to Theorem \ref{theorem1}. Therefore, all distributions of source and target domains match each other well in the common label set.

\begin{figure}
\centering
\includegraphics[width=1\linewidth]{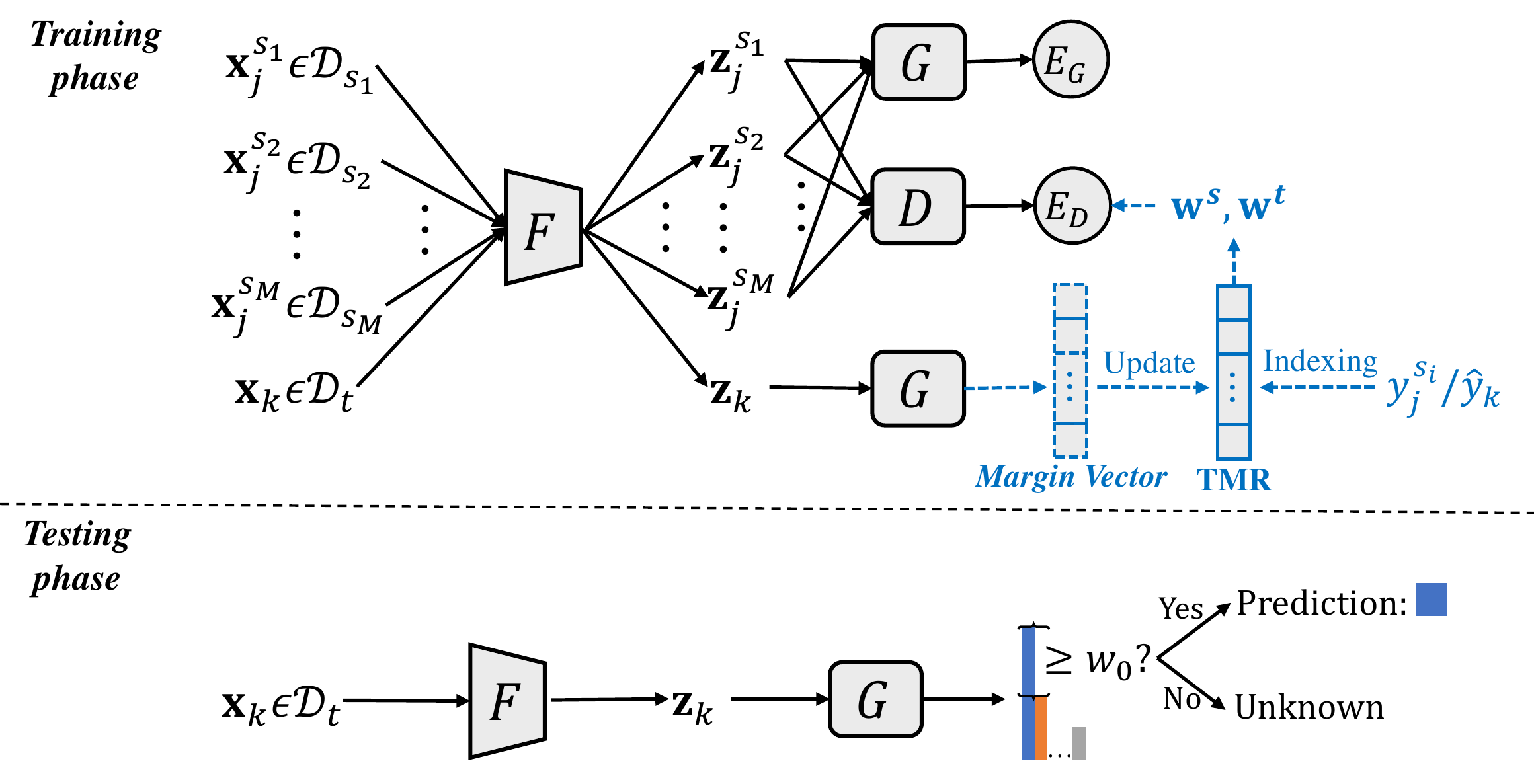}
\caption{The training and testing phases of proposed \textbf{UMAN}.}
\label{model}
\end{figure}

\textbf{Optimization.} The whole training procedure can be written as:
\begin{equation}
\begin{aligned}
\max _D \min _{F, G} E_{G}-E_{D}.
\label{optimization}
\end{aligned}
\end{equation} This optimization objective is very simple but effective. We use the gradient reversal layer \cite{ganin2016domain} to reverse the gradient between $F$ and $D$, this allows the optimization of all modules in an end-to-end way. Note that both $E_G$ and $E_D$ calculate expectations on multiple source domains. Even if multiple source domains are merged into a mixed source domain, UMAN still works. This means that UMAN can complete the domain adaptation process in the case of the domain gap among multiple source domains.

\textbf{Inference.} Finally, a target test sample $\mathbf{x}$ is assigned to either corresponding known class of the source domains or the unknown class:
\begin{equation}
\begin{aligned}
y(\mathbf{x})=\left\{\begin{array}{ll}
{\arg \max}_{c\in \mathcal{C}_s} G^{c}(F(\mathbf{x})), & m_{G\circ F}(\mathbf{x})\geq w_0 \\
\text {unknown}, & \text { otherwise }.
\end{array}\right.
\label{inference}
\end{aligned}
\end{equation}where $m_{G\circ F}(\mathbf{x})$ is the \emph{margin} of the test sample $\mathbf{x}$ calculated by the entity of classifier $G\circ F$ according to Eq. \ref{margin}, $w_0$ is the threshold for separating unknown-class samples from known-class samples. The whole model is shown in Fig. \ref{model}, the algorithm of UMAN is shown in Algorithm \ref{alg}, and the theoretical guarantees of UMAN are provided in Section \ref{Theoretical Guarantees for UMAN}. Note that in UMAN algorithm, we update $\mathbf{V}_{TMR}$ only when the max source error lower than a threshold $\epsilon$.

\begin{algorithm}[h]
\caption{UMAN Algorithm}
\label{alg}
\LinesNumbered
\KwIn{$T$: max iteration\;$w_0$: the threshold for separating unknown-class samples from known-class samples\;
$\mathcal{D}_{s_i}=\{\textbf{x}^{s_i}_{j},y^{s_i}_j\}_{j=1}^{n_{s_i}},i=1,2,\cdots,M$: source training datasets\; 
$\mathcal{D}_t=\{ \textbf{x}_{k}\}_{k=1}^{n_t}$: target training dataset\; 
$F$: pretrained feature extractor parameterized by $\theta _f$\; 
$G$: randomly initialized classifier parameterized by $\theta _g$\; 
$D$: randomly initialized domain classifier parameterized by $\theta _d$\; 
}
\KwOut{$\theta_f $, $\theta _g $ and $\theta _d$.}
\textbf{Training}:\\
set $t=0,\epsilon=0.1$ and initialize $V_{TMR}=V^0_{TMR}$\;
\While{$t<T$}{
\For{each batch $(\mathcal{D}^{b}_{s_1},\mathcal{D}^{b}_{s_2},\cdots,\mathcal{D}^{b}_{s_M}, \mathcal{D}^{b}_{t})$ in $(\mathcal{D}_{s_1},\mathcal{D}_{s_2},\cdots,\mathcal{D}_{s_M}, \mathcal{D}_t)$}{
obtain extracted features on each domain: $\textbf{z}^{s_i}_j=F(\textbf{x}^{s_i}_j)$ and $\textbf{z}_k=F(\textbf{x}_k)$\;
obtain classification results using softmax: $G(\textbf{z}^{s_i}_j)$ and $G(\textbf{z}_k)$\;
calculate the margin $m_{G\circ F}(\textbf{x}_k)$ by Eq. \ref{margin} and \ref{pseudo}\;
obtain the \emph{margin vector} $\mathcal{M}(\mathcal{D}^{b}_{t},F\circ G)$ by Eq. \ref{mv}\;
calculate source error ($\mathbf{1}\{\cdot\}=1$ if $\{\cdot\}$ is true): $error(\mathcal{D}_{s_i})=\frac{1}{n_{s_i}}\sum_{j=1}^{n_{s_i}}\mathbf{1}\{G(F(\mathbf{x}_j))\neq y_j\}$\;
\If{$\max_i error(\mathcal{D}_{s_i})<\epsilon$}{update TMR by Eq. \ref{updated rule} and \ref{init}\;}
weight $\textbf{x}^{s_i}_j$ using TMR-vector by Eq. \ref{wsv}\;
weight $\textbf{x}_k$ using TMR-vector and \emph{margin} by Eq. \ref{wtv}\;
calculate $E_{G}$ by Eq. \ref{Lc} and $E_{D}$ by Eq. \ref{Ld}\;
update $\theta_f $, $\theta _g $ and $\theta _d$ by optimizing Eq. \ref{optimization} using statistical gradient descent\;
let $t\leftarrow t+1$\;}}
\textbf{Testing}:\\
Test data $\textbf{x}$ is forwarded to obtain $G^c(F(\textbf{x}))$ over $c\in\mathcal{C}_s$\;
\eIf{$m_{G\circ F}(\textbf{x})\geq w_0$}{
label $\textbf{x}$ with $y=\arg\max_c G^c(F(\textbf{x}))$\;}
{reject $\textbf{x}$ as an unknown class\;}
\end{algorithm}

\subsection{Theoretical Guarantees of UMAN}
\label{Theoretical Guarantees for UMAN}
In this section, we provide theoretical analysis of the proposed UMAN and show how it works. As shown in Eq. \ref{optimization}, the objective function of UMAN consists of two parts: $E_G$ and $E_D$. For samples drawn from $M$ source domains, UMAN employs the same category classifier $C$ and domain classifier $D$, we have the following theorems:

\begin{theorem}
Let $p_{ij}$ be the distribution of the $i$-th source domain with labels in the label set $\mathcal{C}_{s_i}\cap\mathcal{C}_{s_j}$, $\overline{p}_{i}$ be the distribution of the $i$-th source domain in the private label set. Let $G'$ be an auxiliary classifier with the same structure as $G$. Minimize $E_G$ with respect to $F$ and $G$ is equivalent to 
\begin{equation}
\begin{aligned}
\max _{G'} \min _{F,G} -\frac{1}{M}\sum_{i=1}^{M-1} \sum_{j=i+1}^{M}&\mathbb{E}_{(\mathbf{x}, y) \sim p_{ij}}\log G'^y(F(\mathbf{x}))\\
-&\mathbb{E}_{(\mathbf{x}, y) \sim p_{ji}}\log (1-G'^y(F(\mathbf{x})))\\
-\frac{1}{M}\sum_{i=1}^M&\mathbb{E}_{(\mathbf{x}, y) \sim \overline{p}_i} \log G^y(F(\mathbf{x})).
\label{theorem1.0}
\end{aligned}
\end{equation}
\label{theorem1}
\end{theorem}

\begin{proof}
For simplicity, we only prove Theorem \ref{theorem1} when the label sets of every two source domains are disjoint. If they intersect, Theorem \ref{theorem1} can also be proved similarly. We divide $E_G$, defined in Eq. \ref{Lc}, into three parts:
\begin{equation}
\begin{aligned}
E_G=-\frac{1}{M}\sum_{i=1}^{M-1} \sum_{j=i+1}^{M}&\mathbb{E}_{(\mathbf{x}, y) \sim p_{ij}}\log G^y(F(\mathbf{x}))\\
-&\mathbb{E}_{(\mathbf{x}, y) \sim p_{ji}}\log G^y(F(\mathbf{x}))\\
-\frac{1}{M}\sum_{i=1}^M&\mathbb{E}_{(\mathbf{x}, y) \sim \overline{p}_i} \log G^y(F(\mathbf{x})).
\label{theorem1.1}
\end{aligned}
\end{equation}

The gradients of the first two parts have the same direction when minimizing $E_G$ for $G$ and $F$. And the optimal solution for both two parts is
\begin{equation}
\begin{aligned}
(G\circ F)^*(x)=\{G\circ F|\forall(\mathbf{x}, y) \sim p_{ij},G^y(F(\mathbf{x}))=1\},
\label{G*}
\end{aligned}
\end{equation}for all $i,j=1,\cdots,M$ and $i\neq j$. When $G\circ F=(G\circ F)^*$, the distributions of $F(\mathbf{x}\sim p_{ij})$ and $F(\mathbf{x}\sim p_{ji})$ are forced to be indistinguishable for $G$. Therefore, the first two parts of Eq. \ref{theorem1.1} can be taken as an equivalent objective for $F$:
\begin{equation}
\begin{aligned}
\max _{G'} \min _{F} -\frac{1}{M}\sum_{i=1}^{M-1} \sum_{j=i+1}^{M}&\mathbb{E}_{(\mathbf{x}, y) \sim p_{ij}}\log G'^y(F(\mathbf{x}))\\
-&\mathbb{E}_{(\mathbf{x}, y) \sim p_{ji}}\log (1-G'^y(F(\mathbf{x}))).
\end{aligned}
\end{equation}

Here, $G'$ plays a role of a discriminator to force $F$ to extract indistinguishable features of samples drawn from $p_{ij}$ and $p_{ji}$. Hence, the distributions of any two source domains are aligned in their common label set. 
\end{proof}

\begin{theorem}
The max-min of $E_D$ with respect to $D$ and $F$ is equivalent to minimizing the Jensen-Shannon Divergence between the center of the distributions of source domains $\mathcal{D}_{s_i}=\left\{\left(\mathbf{x}_{j}^{s_i}, y_{j}^{s_i}\right)\right\},i=1,2,\cdots,M$ and the distribution of target domain $\mathcal{D}_{t}=\left\{\left(\mathbf{x}_{k}^{t}\right)\right\}$ in the union common label set $\mathcal{C}=\bigcup_{i}\mathcal{C}_i$, i.e. 
\begin{equation}
\begin{aligned}
&\max _D \min _{F} -E_{D}\\
\iff&\min _{F} \mathrm{JSD}(\frac{1}{M}\sum_{i=1}^MF(p_{\mathcal{C}_i}(\mathbf{x}_j)) \| F(q_{\mathcal{C}}(\mathbf{x}_k))).
\end{aligned}
\end{equation}
\label{theorem2}
\end{theorem}

\begin{proof}
Here, we consider that the ideal distribution of weights can be approximated by Eq. \ref{wsv} and Eq. \ref{wtv}, i.e., 
\begin{equation}
\begin{aligned}
w^{s}(y_j)\rightarrow\left\{\begin{array}{ll}
1, & y_j\in \mathcal{C} \\
0, & y_j\notin \mathcal{C},
\end{array}\right.
\label{wsi*}
\end{aligned}
\end{equation}
\begin{equation}
\begin{aligned}
w^{t}(\mathbf{x}_k,f)\rightarrow\left\{\begin{array}{ll}
1, & y_k\in \mathcal{C} \\
0, & y_k\notin \mathcal{C}.
\end{array}\right.
\label{wsi*}
\end{aligned}
\end{equation}

Then, $E_D$ defined in Eq. \ref{Ld} can be written as
\begin{equation}
\begin{aligned} 
E_{D}(F,D)\rightarrow-\frac{1}{M}\sum_{i=1}^M&\mathbb{E}_{(\mathbf{x}, y) \sim p_{\mathcal{C}_i}}\log D\left(F(\mathbf{x})\right)
\\ -\mathbb{E}_{\mathbf{x} \sim q_{\mathcal{C}}}&\log \left(1-D\left(F(\mathbf{x})\right)\right),
\label{Ld-}
\end{aligned}
\end{equation}

\begin{figure*}[h]
\centering
\includegraphics[height=1.6in,width=1\linewidth]{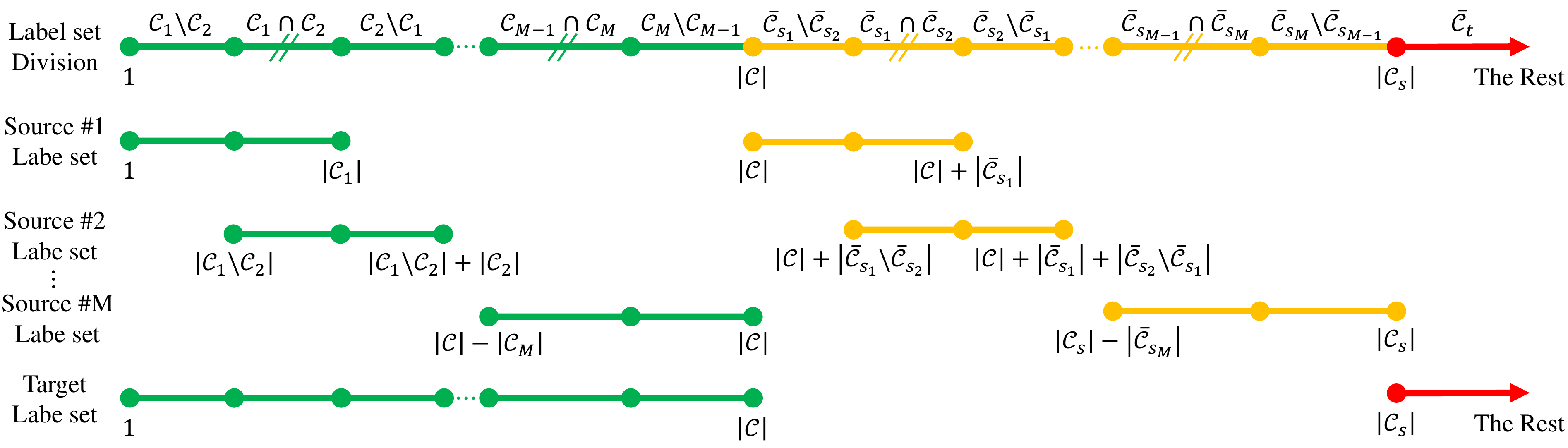}
\caption{The division rule of the common label set (green), the source private label set (orange) and the target private label set (red). Note that the lengths of the intersections are as equal as possible in the common label set (green double slash) and the source private label set (orange double slash) respectively.}
\label{labelsets}
\end{figure*}

Given $\mathbf{z}=F(x)$, we have the optimal solution when maximizing $-E_D$ for $D$:
\begin{equation}
\begin{aligned}
D^*(\mathbf{z})\rightarrow\frac{\frac{1}{M}\sum_{i=1}^Mp_{\mathcal{C}_i}(\mathbf{z})}{\frac{1}{M}\sum_{i=1}^Mp_{\mathcal{C}_i}(\mathbf{z})+q_{\mathcal{C}}(\mathbf{z})},
\label{D*}
\end{aligned}
\end{equation}where $p_{\mathcal{C}_i}(\mathbf{z})$ and $q_{\mathcal{C}_i}(\mathbf{z})$ are the probabilities that $\mathbf{x}=F^{-1}(\mathbf{z})$ can be sampled from $p_{\mathcal{C}_i}$ and $q_{\mathcal{C}_i}$ respectively. Optimal value can be calculated as follows:
\begin{equation}
\begin{aligned} 
-E_{D}(F,D^*)\rightarrow\frac{1}{M}\sum_{i=1}^M\mathbb{E}_{(\mathbf{x}, y) \sim p_{\mathcal{C}_i}}\frac{\frac{1}{M}\sum_{i=1}^Mp_{\mathcal{C}_i}(\mathbf{z})}{\frac{1}{M}\sum_{i=1}^Mp_{\mathcal{C}_i}(\mathbf{z})+q_{\mathcal{C}}(\mathbf{z})}&\\ 
+\mathbb{E}_{\mathbf{x} \sim q_{\mathcal{C}_i}}\frac{q_{\mathcal{C}_i}(\mathbf{z})}{\frac{1}{M}\sum_{i=1}^Mp_{\mathcal{C}_i}(\mathbf{z})+q_{\mathcal{C}}(\mathbf{z})}&\\
=\int_{f}\frac{1}{M}\sum_{i=1}^Mp_{\mathcal{C}_i}(\mathbf{z})\frac{\frac{1}{M}\sum_{i=1}^Mp_{\mathcal{C}_i}(\mathbf{z})}{\frac{1}{M}\sum_{i=1}^Mp_{\mathcal{C}_i}(\mathbf{z})+q_{\mathcal{C}}(\mathbf{z})}d\mathbf{z}&\\ 
+\int_{\mathbf{z}}q_{\mathcal{C}_i}(f)\frac{q_{\mathcal{C}_i}(\mathbf{z})}{\frac{1}{M}\sum_{i=1}^Mp_{\mathcal{C}_i}(\mathbf{z})+q_{\mathcal{C}}(\mathbf{z})}d\mathbf{z}&\\
=-2 \log 2+\int_{\mathbf{z}}\frac{1}{M}\sum_{i=1}^Mp_{\mathcal{C}_i}(f)\frac{\frac{1}{M}\sum_{i=1}^Mp_{\mathcal{C}_i}(\mathbf{z})}{\frac{1}{2M}\sum_{i=1}^Mp_{\mathcal{C}_i}(\mathbf{z})+\frac{1}{2}q_{\mathcal{C}}(\mathbf{z})}d\mathbf{z}&\\ 
+\int_{\mathbf{z}}q_{\mathcal{C}_i}(f)\frac{q_{\mathcal{C}_i}(\mathbf{z})}{\frac{1}{2M}\sum_{i=1}^Mp_{\mathcal{C}_i}(\mathbf{z})+\frac{1}{2}q_{\mathcal{C}}(\mathbf{z})}d\mathbf{z}&\\
=-2 \log 2+\mathrm{KL}(\frac{1}{M}\sum_{i=1}^Mp_{\mathcal{C}_i} \| \frac{\frac{1}{M}\sum_{i=1}^Mp_{\mathcal{C}_i}+q_{\mathcal{C}}}{2})&\\
+\mathrm{KL}(q_{\mathcal{C}_i} \| \frac{\frac{1}{M}\sum_{i=1}^Mp_{\mathcal{C}_i}+q_{\mathcal{C}}}{2})&\\
=-2\log 2+2\mathrm{JSD}(\frac{1}{M}\sum_{i=1}^Mp_{\mathcal{C}_i} \| q_{\mathcal{C}})&,
\label{Ld-}
\end{aligned}
\end{equation}where $\mathrm{KL}(\cdot\|\cdot)$ is the Kullback-Leibler Divergence and $\mathrm{JSD}(\cdot\|\cdot)$ is the Jensen-Shannon Divergence. Then, with minimizing $-E_D$ for $F$ with a much smaller learning rate, we get an equivalent optimization:
\begin{equation}
\min _{F} \mathrm{JSD}(\frac{1}{M}\sum_{i=1}^MF(p_{\mathcal{C}_i}(\mathbf{x}_j)) \| F(q_{\mathcal{C}}(\mathbf{x}_k))),
\label{eqop}
\end{equation}where $\mathbf{x}_j$ are drawn from $p_{\mathcal{C}_i}$ and $\mathbf{x}_k$ are drawn from $q_{\mathcal{C}}$. Hence, the center of the distributions of source domains and the distribution of target domain are aligned in their common label set.
\end{proof}

\section{Experiments}
\subsection{Experimental Setup}
\subsubsection{Label Sets}
The division rule of all the label sets introduced in Section \ref{Problem Setting} is shown in Fig. \ref{labelsets}. Each domain contains two types of label sets: common and private. Therefore, we use a matrix to describe a specific UMDA setting, called \textbf{UMDA-Matrix}, which is defined as
\begin{equation}
\left[{\begin{array}{c:c}
\begin{matrix}
|\mathcal{C}_1| & |\mathcal{C}_2| & \cdots & |\mathcal{C}_M| \\
|\overline{\mathcal{C}}_{s_1}| & |\overline{\mathcal{C}}_{s_2}| & \cdots & |\overline{\mathcal{C}}_{s_M}|
\end{matrix}&
\begin{matrix}
|\mathcal{C}| \\
|\overline{\mathcal{C}}_{t}|
\end{matrix}
\end{array}}\right].
\label{vUMDA}
\end{equation} The first line of Eq. \ref{vUMDA} is the size of the common label set of all the domains, and the second line the private label set. The first $M$ columns of Eq. \ref{vUMDA} are the label set of the source domains, and the last column the target domain. In this way, UMDA settings can be determined by the division rule and the UMDA-Matrix.

\subsubsection{Datasets}
The UMDA settings on each dataset are shown in Table \ref{UMDAsettings}, where multiple UMDA-Matrices mean that multiple UMDA settings are considered in one dataset. More settings are explored in Section \ref{UMAN Settings}. The following is additional information about the datasets:

\textbf{Office-31} \cite{saenko2010adapting} dataset contains 31 categories in 3 different domains: \textbf{A}, \textbf{D}, \textbf{W}. There are at most two source domains and one target domain, i.e. $M=2$. We use the 10 classes shared by Office-31 and Caltech-256 \cite{gong2012geodesic} as the total common label set $\mathcal{C}$, and the rest of the classes are sorted in alphabetical order.

\textbf{Office-Home} \cite{venkateswara2017deep} is a larger dataset with 65 classes in 4 different domains: Artistic (\textbf{Ar}), Clip-Art (\textbf{Cl}), Product (\textbf{Pr}) and Real-World ($\mathbf{Rw}$) images. In this dataset, $M$ can be 2 and 3. All the classes are sorted in alphabetical order.

\textbf{VisDA2017+ImageCLEF-DA}. VisDA2017 \cite{peng2018visda} dataset focuses on a special domain adaptation setting (simulation to real). The source domain consists of simulated images ($\textbf{S}$) generated by game engines and target domain consists of real-world images ($\textbf{R}$). ImageCLEF-DA\footnote{https://www.imageclef.org/2014/adaptation} is a benchmark dataset for ImageCLEF 2014 domain adaptation challenge, which is organized by selecting the common categories shared by the following three public datasets. Here, each dataset is considered as a domain: Caltech-256 (\textbf{C}), ImageNet ILSVRC 2012 (\textbf{I}), and Pascal VOC 2012 (\textbf{P}). In this dataset, $M$ can be 2, 3 and 4. Classes are numbered in this way: No. 1-7 classes are the shared classes of five datasets, in alphabet order, No. 8-12 classes are the remaining classes of \textbf{S} and \textbf{R}, and No. 13-17 classes are the remaining classes of \textbf{C}, \textbf{I} and \textbf{P}.

\newcommand{\tabincell}[2]{\begin{tabular}{@{}#1@{}}#2\end{tabular}}
\renewcommand\arraystretch{1.1}
\begin{table}
\centering \small
\caption{The UMDA settings on different datasets.}
\setlength{\tabcolsep}{0.6mm}{
\begin{tabular}{|c|c|c|c|c|c|}
\hline
Datasets & classes & Domains & Images & UMDA-Matrices \\
\hline
\multirow{3}{*}{Office-31}&\multirow{3}{*}{31}& $\mathrm{A}$ & 2817 & \multirow{3}{*}{$\left[{\begin{array}{c:c}
\begin{matrix}
7 & 7\\
5 & 5
\end{matrix}&
\begin{matrix}
10 \\
11
\end{matrix}
\end{array}}\right]$} \\
\cline{3-4}&& $\mathrm{D}$ & 498 &\\
\cline{3-4}&& $\mathrm{W}$ & 795 &\\
\hline
\multirow{4}{*}{Office-Home}&\multirow{4}{*}{65}& $\mathrm{Ar}$ & 2427 & \multirow{2}{*}{$\left[{\begin{array}{c:c}
\begin{matrix}
7 & 7\\
5 & 5
\end{matrix}&
\begin{matrix}
10 \\
50
\end{matrix}
\end{array}}\right]$} \\
\cline{3-4}&& $\mathrm{Cl}$ & 4365 &\\
\cline{3-4}&& $\mathrm{Pr}$ & 4439 &\multirow{2}{*}{$\left[{\begin{array}{c:c}
\begin{matrix}
4 & 4 & 4\\
2 & 2 & 2
\end{matrix}&
\begin{matrix}
10 \\
50
\end{matrix}
\end{array}}\right]$}\\
\cline{3-4}&& $\mathrm{Rw}$ & 4357 &\\
\hline
\multirow{6}{*}{\tabincell{c}{VisDA2017\\+\\ImageCLEF-DA}}&\multirow{6}{*}{17}& \multirow{2}{*}{$\mathrm{S}$} & \multirow{2}{*}{152397} & \multirow{2}{*}{$\left[{\begin{array}{c:c}
\begin{matrix}
4 & 4\\
1 & 1
\end{matrix}&
\begin{matrix}
7 \\
3
\end{matrix}
\end{array}}\right]$} \\
&&&&\\
\cline{3-4}&& $\mathrm{R}$ & 55388 &\multirow{2}{*}{$\left[{\begin{array}{c:c}
\begin{matrix}
3 & 3 & 3\\
1 & 1 & 1
\end{matrix}&
\begin{matrix}
7 \\
2
\end{matrix}
\end{array}}\right]$}\\
\cline{3-4}&& $\mathrm{C}$ & 600 &\\
\cline{3-4}&& $\mathrm{I}$ & 600 &\multirow{2}{*}{$\left[{\begin{array}{c:c}
\begin{matrix}
2 & 2 & 2 & 1\\
1 & 1 & 1 & 1
\end{matrix}&
\begin{matrix}
7 \\
1
\end{matrix}
\end{array}}\right]$}\\
\cline{3-4}&& $\mathrm{P}$ & 600 &\\
\hline
\end{tabular}}
\label{UMDAsettings}
\end{table}

Some samples in each dataset are shown in Fig. \ref{datasets}.
\begin{figure}
\centering
\includegraphics[height=2.2in,width=1\linewidth]{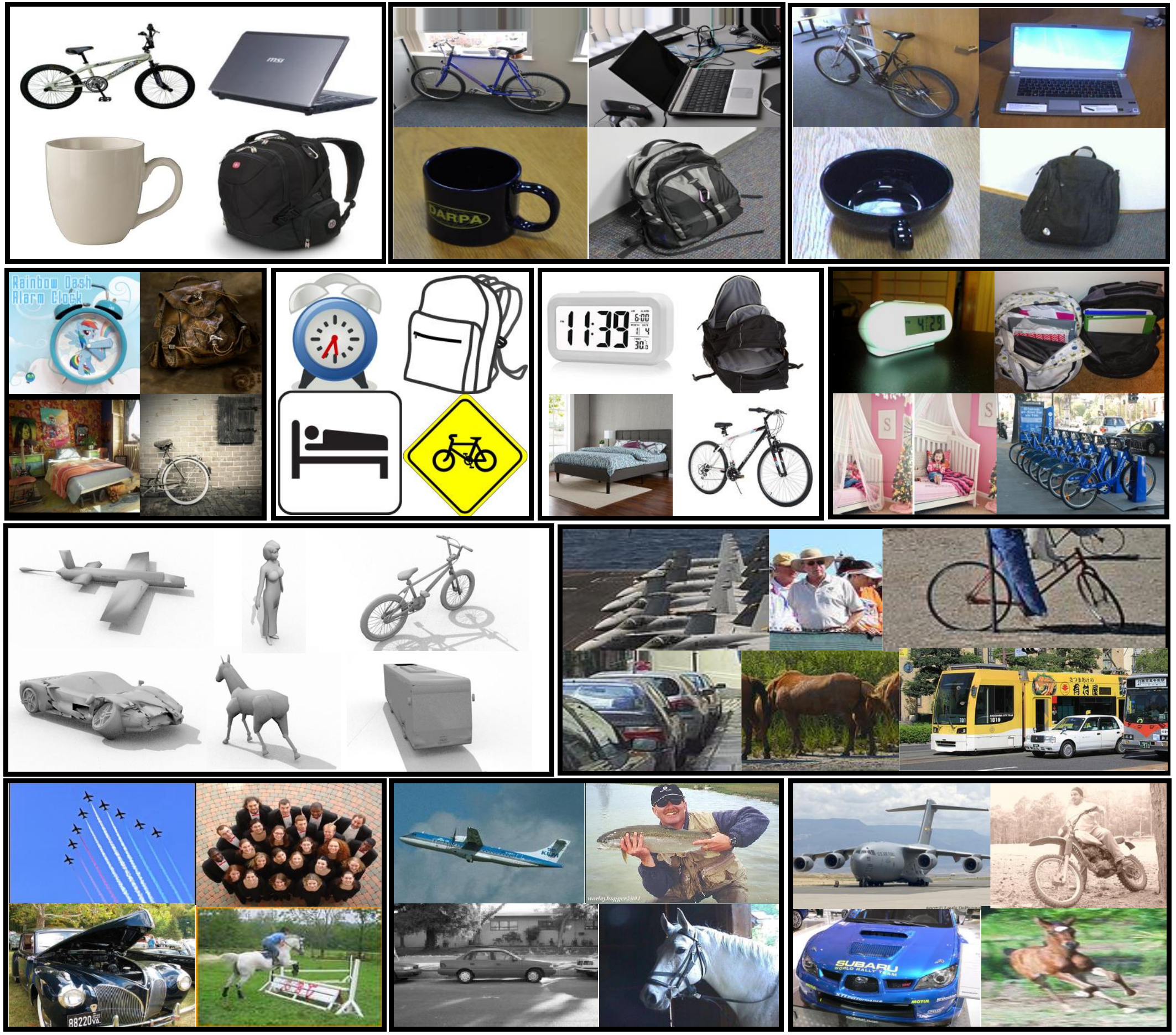}
\caption{Some samples in following datasets (from top to bottom and left to right). (1) \textbf{Office-31}: Amazon (\textbf{A}), Dslr (\textbf{D}) and Webcam (\textbf{W}). \textbf{(2) Office-Home}: Artistic (\textbf{Ar}), Clip-Art (\textbf{Cl}), Product (\textbf{Pr}) and Real-World images (\textbf{W}). \textbf{(3) VisDA2017}: Synthetic (\textbf{Syn}) and Real images (\textbf{Real}). \textbf{(4) ImageCLEF-DA}: Caltech-256 (\textbf{C}), ImageNet ILSVRC 2012 (\textbf{I}), and Pascal VOC 2012 (\textbf{P}).}
\label{datasets}
\end{figure}

\subsubsection{Evaluation Details}
\textbf{Compared Methods}. The proposed UMAN is compared with \textbf{(1)} Single-source-only without domain adaptation (DA): \textbf{ResNet} \cite{he2016deep}, \textbf{(2)} closed set single-source DA: Domain-Adversarial Neural Networks (\textbf{DANN}) \cite{ganin2016domain}, Residual Transfer Networks (\textbf{RTN}) \cite{long2016unsupervised}, \textbf{(3)} partial single-source DA: Importance Weighted Adversarial Nets (\textbf{IWAN}) \cite{zhang2018importance}, Partial Adversarial Domain Adaptation (\textbf{PADA}) \cite{cao2018partialeccv}, \textbf{(4)} open set single-source DA: Open Set Back-Propagation (\textbf{OSBP}) \cite{saito2018open}, \textbf{(5)} universal single-source DA: Universal Adaptation Network (\textbf{UAN}) \cite{you2019universal}, \textbf{(6)} closed set multi-source DA: Multisource Domain Adversarial Networks (\textbf{MDAN}) \cite{zhao2018adversarial}, Multi-source Distilling Domain Adaptation (\textbf{MDDA}) \cite{ZhaoWZGLS0HCK20}, \textbf{(7)} closed set multi-source DA with category shift: Deep CockTail Network (\textbf{DCTN}) \cite{xu2018deep}. These methods are state of the art in their respective settings. In the following experiments, we evaluation all these methods in the UMDA setting.

\textbf{Evaluation Standards}. We evaluate these domain adaptation methods under four standards. \textbf{(1) Source-only}: train on the source domains and directly test on the target domain. \textbf{(2) Single-best UDA}: train on each source domain and adapt to target domain, in the source shared common label sets $\mathcal{C}_i\cap\mathcal{C}_j$, take the highest accuracy adapted from the $i$-th and $j$-th source domain, in the source specific common label sets $\mathcal{C}_j\setminus(\mathcal{C}_i\cap\mathcal{C}_j)$, take the accuracy adapted from the $j$-th source domain. \textbf{(3) Source-combined UDA}: combine all the source domains as a new source domain, train on the new source domain and adapt to target domain. \textbf{(3) UMDA}: train on multiple source domains and adapt to target domain, here domain discrepancy and label discrepancy are present between every two domains.

\textbf{Evaluation Protocols}. All the samples with their labels in $\overline{\mathcal{C}}_t$ are viewed as one unified "unknown" class and the final accuracy is averaged by per-class accuracy for $|\mathcal{C}|+1$ classes. At the testing stage, when the maximum probability of the prediction $\max_{c\in \mathcal{C}_s}G^c(F(\mathbf{x}))$ is less than $w_0$, the input image is classified as "unknown", otherwise the input image is classified as $y={\arg\max}_{c\in \mathcal{C}_s}G^c(F(\mathbf{x}))$.

\textbf{Implementation Details}. Implementation is in PyTorch and ResNet-50 \cite{he2016deep} is used as fine-tuned model, which is pre-trained on ImageNet.

\textbf{Hyperparameters}. We adjust the threshold $w_0$ on $\bm{\mathrm{DW}\rightarrow\mathrm{A}}$. As shown in Table \ref{w_0}, we observe that UMAN is not sensitive to $w_0$ and $w_0=0.5$ is always employed in all experiments.

\begin{table}[h]
\centering 
\caption{The effect of $w_0$ on the test accuracy (\%) for $\mathrm{DW}\rightarrow\mathrm{A}$.} 
\begin{tabular}{ccccccc} 
\toprule 
\multicolumn {7}{c}{$w_0$}\\
\midrule 
0.2 & 0.3 & 0.4 & 0.5 & 0.6 & 0.7 & 0.8 \\ 
\midrule 
90.66 & 91.65 & 91.58 & 90.45 & 90.19 & 89.75 & 88.74 \\
\bottomrule 
\end{tabular} 
\label{w_0} 
\end{table}

\subsection{Classification Results}
\renewcommand\arraystretch{0.8}
\begin{table}
\centering \small
\caption{Comparison with the state-of-the-art DA methods of universal multi-source domain adaptation tasks on \textbf{Office-31} ($M=2$, $\xi_1=\xi_2=0.30$ and $\xi_{12}=0.40$) dataset (Backbone: ResNet-50) measured by classification accuracy (\%).}
\setlength{\tabcolsep}{2.1mm}{
\begin{tabular}{c|c|c|c|c|c}
\hline
\multirow{3}{*}{Standards} & Methods & \multicolumn {3}{c|}{Office-31}& \multirow{3}{*}{Avg}\\
\cline{2-5}&Source& $\mathrm{DW}$ & $\mathrm{AW}$ & $\mathrm{AD}$ \\
\cline{2-5}&Target& $\mathrm{A}$ & $\mathrm{D}$ & $\mathrm{W}$ \\
\hline
\multirow{2}{*}{Source-only} & Combined & 83.28 & 88.64 & 83.78 & 85.23 \\
& Single-best & 80.64 & 87.77 & 85.50 & 84.64 \\
\hline
\multirow{7}{*}{\tabincell{c}{Single-best\\UDA}} 
& DANN \cite{ganin2016domain} & 80.92 & 86.45 & 80.85 & 82.74 \\
& RTN \cite{long2016unsupervised} & 80.67 & 87.04 & 87.17 & 84.96 \\
& IWAN \cite{zhang2018importance} & 85.64 & 88.28 & 88.64 & 87.52 \\
& PADA \cite{cao2018partialeccv} & 74.42 & 88.14 & 83.54 & 82.03 \\
& OSBP \cite{saito2018open} & 56.54 & 81.81 & 71.34 & 69.90 \\
& UAN \cite{you2019universal} & 85.35 & 94.54 & 92.03 & 90.65 \\
\hline
\multirow{7}{*}{\tabincell{c}{Source-\\combined\\UDA}} 
& DANN \cite{ganin2016domain} & 83.43 & 81.36 & 84.22 & 83.00 \\
& RTN \cite{long2016unsupervised} & 86.70 & 88.64 & 83.22 & 86.19 \\
& IWAN \cite{zhang2018importance} & 85.64 & 87.18 & 82.31 & 85.04 \\
& PADA \cite{cao2018partialeccv} & 88.15 & 89.19 & 86.09 & 87.81 \\
& OSBP \cite{saito2018open} & 57.76 & 81.54 & 78.48 & 72.59 \\
& UAN \cite{you2019universal} & 82.97 & 88.64 & 81.23 & 84.28 \\
\hline
\multirow{4}{*}{\tabincell{c}{UMDA}} 
& DCTN \cite{xu2018deep} & 88.84 & 89.18 & 83.73 & 87.25 \\
& MDAN \cite{zhao2018adversarial} & 85.82 & 92.82 & 88.43 & 89.02 \\
& MDDA \cite{ZhaoWZGLS0HCK20} & 85.40 & 88.50 & 82.40 & 85.43 \\
& \bf{UMAN (ours)} & \bf{90.45} & \bf{95.12} & \bf{92.61} & \bf{92.73} \\
\hline
\end{tabular}}
\label{Office-31}
\end{table}

\begin{table*}[h]
\centering \small
\caption{Comparison with the state-of-the-art DA methods of universal multi-source domain adaptation tasks on \textbf{Office-Home} ($M=2$, $\xi_1=\xi_2=0.12$ and $\xi_{12}=0.40$) dataset (Backbone: ResNet-50) measured by classification accuracy (\%).}
\setlength{\tabcolsep}{1.5mm}{
\begin{tabular}{c|c|c|c|c|c|c|c|c|c|c|c|c|c|c}
\hline
\multirow {3}{*}{Standards} & Methods & \multicolumn {12}{c|}{Office-Home $(M=2)$} & \multirow {3}{*}{Avg} \\
\cline{2-14}&Source& $\mathrm{ClPr}$ & $\mathrm{ClRw}$ & $\mathrm{PrRw}$ & $\mathrm{ArPr}$ & $\mathrm{ArRw}$ & $\mathrm{PrRw}$ & $\mathrm{ArCl}$ & $\mathrm{ArRw}$ & $\mathrm{ClRw}$ & $\mathrm{ArCl}$ & $\mathrm{ArPr}$ & $\mathrm{ClPr}$ \\
\cline{2-14}&Target& $\mathrm{Ar}$ & $\mathrm{Ar}$ & $\mathrm{Ar}$ & $\mathrm{Cl}$ & $\mathrm{Cl}$ & $\mathrm{Cl}$ & $\mathrm{Pr}$ & $\mathrm{Pr}$ & $\mathrm{Pr}$ & $\mathrm{Rw}$ & $\mathrm{Rw}$ & $\mathrm{Rw}$\\
\hline
\multirow{2}{*}{Source-only} & Combined & 70.06 & 73.04 & 77.82 & 58.09 & 60.20 & 59.23 & 77.23 & 79.82 & 75.50 & 85.53 & 86.79 & 80.48 & 73.65 \\
& Single-best & 72.56 & 76.03 & 77.19 & 58.45 & 59.33 & 58.34 & 74.91 & 77.99 & 76.35 & 85.73 & 87.06 & 84.75 & 74.06 \\
\hline
\multirow{7}{*}{\tabincell{c}{Single-best\\UDA}} 
& DANN \cite{ganin2016domain} & 69.55 & 76.19 & 76.56 & 56.64 & 56.93 & 57.13 & 79.22 & 80.68 & 76.80 & 85.94 & 86.55 & 85.19 & 73.95 \\
& RTN \cite{long2016unsupervised} & 67.04 & 74.98 & 75.80 & 48.89 & 54.02 & 52.45 & 76.48 & 78.49 & 77.17 & 86.35 & 86.48 & 85.37 & 71.96 \\
& IWAN \cite{zhang2018importance} & 73.59 & 75.41 & 76.70 & 55.90 & 57.52 & 58.95 & 78.28 & 80.65 & 76.53 & 86.14 & 86.08 & 85.22 & 74.25 \\
& PADA \cite{cao2018partialeccv} & 58.32 & 71.61 & 67.68 & 38.44 & 43.02 & 41.89 & 68.78 & 75.48 & 74.89 & 77.11 & 78.60 & 78.96 & 64.57 \\
& OSBP \cite{saito2018open} & 60.94 & 67.18 & 67.91 & 46.78 & 52.79 & 51.82 & 61.38 & 70.90 & 71.10 & 76.05 & 79.55 & 77.82 & 65.35 \\
& UAN \cite{you2019universal} & 77.82 & 81.42 & 81.83 & \bf{61.68} & \bf{63.12} & 61.80 & 81.59 & 81.81 & 79.21 & 87.10 & 87.54 & 86.37 & 77.61 \\
\hline
\multirow{7}{*}{\tabincell{c}{Source-\\combined\\UDA}} 
& DANN \cite{ganin2016domain} & 71.26 & 72.75 & 76.77 & 56.62 & 59.57 & 57.67 & 80.39 & 83.04 & 81.87 & 86.93 & 88.39 & 85.89 & 75.10 \\
& RTN \cite{long2016unsupervised} & 71.63 & 76.48 & 80.03 & 55.79 & 55.36 & 54.51 & 80.53 & 79.87 & 80.94 & 86.92 & 87.47 & 80.53 & 74.17 \\
& IWAN \cite{zhang2018importance} & 70.02 & 71.45 & 74.78 & 52.96 & 55.33 & 51.27 & 71.73 & 82.12 & 73.91 & 85.62 & 87.77 & 84.37 & 71.78 \\
& PADA \cite{cao2018partialeccv} & 71.55 & 72.25 & 76.78 & 55.06 & 59.39 & 54.09 & 78.28 & 79.90 & 73.81 & 85.64 & 88.51 & 83.00 & 73.19 \\
& OSBP \cite{saito2018open} & 58.37 & 66.09 & 65.92 & 47.39 & 53.90 & 52.38 & 64.26 & 68.12 & 62.80 & 80.13 & 80.76 & 73.48 & 64.47 \\
& UAN \cite{you2019universal} & 75.58 & 77.83 & 80.66 & 59.13 & 62.53 & 62.01 & 80.94 & 83.63 & 82.60 & 87.23 & 87.89 & 85.85 & 77.16 \\
\hline
\multirow{4}{*}{\tabincell{c}{UMDA}} 
& DCTN \cite{xu2018deep} & 70.60 & 70.42 & 67.73 & 55.53 & 56.33 & 48.17 & 75.26 & 74.07 & 68.88 & 79.71 & 84.85 & 82.32 & 69.49 \\
& MDAN \cite{zhao2018adversarial} & 72.29 & 71.98 & 76.09 & 58.19 & 62.59 & 58.35 & 78.54 & 80.84 & 75.58 & 87.18 & 88.08 & 85.09 & 74.57 \\
& MDDA \cite{ZhaoWZGLS0HCK20} & 68.30 & 69.89 & 73.48 & 51.79 & 58.30 & 53.97 & 76.68 & 79.01 & 75.19 & 84.69 & 87.03 & 81.26 & 71.63 \\
& \bf{UMAN (ours)} & \bf{82.93} & \bf{83.03} & \bf{82.23} & 60.58 & 61.95 & \bf{62.06} & \bf{84.27} & \bf{84.71} & \bf{84.73} & \bf{88.67} & \bf{90.29} & \bf{87.80} & \bf{79.44} \\
\hline
\end{tabular}}
\label{Office-Home}
\end{table*}

\begin{table}[h]
\centering \small
\caption{Comparison with the state-of-the-art DA methods of universal multi-source domain adaptation tasks on \textbf{Office-Home} ($M=3$, $\xi_1=\xi_2=\xi_3=0.08$, $\xi_{12}=0.00$ and $\xi_{23}=0.33$) dataset (Backbone: ResNet-50) measured by classification accuracy (\%).}
\setlength{\tabcolsep}{0.3mm}{
\begin{tabular}{c|c|c|c|c|c|c}
\hline
\multirow {3}{*}{Standards} & Methods & \multicolumn {4}{c|}{Office-Home $(M=3)$}& \multirow{3}{*}{Avg}\\
\cline{2-6}&Source& $\mathrm{ClPrRw}$ & $\mathrm{ArPrRw}$ & $\mathrm{ArClRw}$ & $\mathrm{ArClPr}$ \\
\cline{2-6}&Target& $\mathrm{Ar}$ & $\mathrm{Cl}$ & $\mathrm{Pr}$ & $\mathrm{Rw}$ \\
\hline
\multirow{2}{*}{Source-only} & Combined & 65.36 & 57.87 & 79.99 & 83.24 & 71.62 \\
& Single-best & 74.55 & 58.70 & 76.29 & 85.75 & 73.82 \\
\hline
\multirow{7}{*}{\tabincell{c}{Single-best\\UDA}} 
& DANN \cite{ganin2016domain} & 73.34 & 56.81 & 78.67 & 85.82 & 73.66 \\
& RTN \cite{long2016unsupervised} & 71.57 & 51.45 & 77.32 & 85.97 & 71.58 \\
& IWAN \cite{zhang2018importance} & 74.63 & 56.82 & 78.32 & 85.71 & 73.87 \\
& PADA \cite{cao2018partialeccv} & 65.08 & 40.79 & 72.41 & 77.96 & 64.06 \\
& OSBP \cite{saito2018open} & 64.51 & 49.98 & 66.82 & 77.30 & 67.06 \\
& UAN \cite{you2019universal} & \bf{79.88} & 62.19 & 80.64 & 86.93 & 77.41 \\
\hline
\multirow{7}{*}{\tabincell{c}{Source-\\combined\\UDA}} 
& DANN \cite{ganin2016domain} & 68.97 & 53.37 & 79.70 & 82.09 & 71.03 \\
& RTN \cite{long2016unsupervised} & 68.72 & 59.97 & 77.04 & 86.00 & 72.93 \\
& IWAN \cite{zhang2018importance} & 62.01 & 48.47 & 77.95 & 82.33 & 67.69 \\
& PADA \cite{cao2018partialeccv} & 62.17 & 50.87 & 73.66 & 81.67 & 67.10 \\
& OSBP \cite{saito2018open} & 44.17 & 45.98 & 63.37 & 68.56 & 55.52 \\
& UAN \cite{you2019universal} & 69.27 & 60.32 & 79.78 & 82.82 & 73.05 \\
\hline
\multirow{4}{*}{\tabincell{c}{UMDA}} 
& DCTN \cite{xu2018deep} & 64.77 & 42.09 & 65.25 & 70.11 & 60.56 \\
& MDAN \cite{zhao2018adversarial} & 67.56 & 55.36 & 79.20 & 86.02 & 72.04 \\
& MDDA \cite{ZhaoWZGLS0HCK20} & 44.66 & 34.54 & 54.93 & 53.24 & 46.84 \\
& \bf{UMAN} & 79.00 & \bf{64.68} & \bf{81.12} & \bf{87.08} & \bf{77.97} \\
\hline
\end{tabular}}
\label{Office-Home3}
\end{table}

\begin{table*}
\centering \small
\caption{Comparison with the state-of-the-art DA methods of universal multi-source domain adaptation tasks on \textbf{VisDA2017+ImageCLEF-DA} ($M=2$, $\xi_1=\xi_2=0.36$ and $\xi_{12}=0.14$) dataset (Backbone: ResNet-50) measured by classification accuracy (\%). (\textbf{Com.}: source-combined. \textbf{Sin.}: single-best.)}
\setlength{\tabcolsep}{0.6mm}{
\begin{tabular}{|c|c|c|c|c|c|c|c|c|c|c|c|c|c|c|c|c|c|c|c|}
\hline
\multirow {2}{*}{$\mathcal{D}_{s_i}$} & \multirow {2}{*}{$\mathcal{D}_{t}$} & \multicolumn {2}{|c|}{Source-only} & \multicolumn {6}{|c|}{Single-best UDA} & \multicolumn {6}{|c|}{Source-combined UDA} & \multicolumn {4}{|c|}{UMDA} \\
\cline{3-20}&& Com. & Sin. & DANN & RTN & IWAN & PADA & OSBP & UAN & DANN & RTN & IWAN & PADA & OSBP & UAN & DCTN & MDAN & MDDA & \bf{UMAN} \\
\hline
$\mathrm{RC}$ & $\mathrm{S}$ & 71.28 & 63.41 & 64.55 & 65.35 & 67.87 & 66.43 & 58.61 & 64.85 & 76.12 & 64.56 & 61.15 & 74.87 & 62.42 & 75.93 & 62.35 & 57.73 & 56.41 & \bf{76.07} \\
$\mathrm{RI}$ & $\mathrm{S}$ & 71.28 & 68.14 & 69.55 & 72.48 & 71.64 & 70.89 & 67.31 & 69.05 & 77.43 & 77.50 & 61.11 & 77.09 & 68.98 & 76.44 & 74.18 & 62.50 & 59.75 & \bf{77.99} \\
$\mathrm{RP}$ & $\mathrm{S}$ & 67.03 & 69.40 & 71.02 & 71.47 & 70.31 & 70.74 & 66.72 & 70.91 & 76.52 & 62.37 & 73.63 & 72.51 & 64.43 & 76.43 & 65.25 & 60.93 & 58.33 & \bf{76.65} \\
$\mathrm{CI}$ & $\mathrm{S}$ & 68.24 & 60.08 & 62.74 & 65.90 & 62.60 & 64.75 & 54.62 & 62.40 & 64.00 & 60.25 & 60.42 & 69.01 & 52.42 & 69.02 & 67.85 & 65.30 & 58.25 & \bf{69.19} \\
$\mathrm{CP}$ & $\mathrm{S}$ & 61.25 & 61.76 & 64.67 & 64.73 & 60.83 & 64.55 & 53.83 & 64.87 & 70.24 & 72.39 & 55.51 & 69.91 & 47.29 & 70.40 & 56.51 & 58.88 & 53.94 & \bf{76.98} \\
$\mathrm{IP}$ & $\mathrm{S}$ & 65.15 & 66.49 & 69.47 & 72.02 & 65.05 & 69.07 & 62.73 & 69.07 & 77.41 & \bf{78.01} & 60.43 & 72.83 & 60.19 & 76.33 & 63.11 & 62.64 & 55.98 & 69.56 \\
\hline
$\mathrm{SC}$ & $\mathrm{R}$ & 61.10 & 64.78 & 67.47 & 67.09 & 57.79 & 66.89 & 49.15 & 67.89 & 69.75 & 64.63 & 58.41 & 60.78 & 40.20 & 65.38 & 67.84 & 63.03 & 37.81 & \bf{74.06} \\
$\mathrm{SI}$ & $\mathrm{R}$ & 61.95 & 67.20 & 67.08 & 67.53 & 65.52 & 66.68 & 52.73 & 68.05 & 70.95 & 63.63 & 55.27 & 63.46 & 41.80 & 71.81 & 67.25 & 59.98 & 32.23 & \bf{74.97} \\
$\mathrm{SP}$ & $\mathrm{R}$ & 57.41 & 67.32 & 67.51 & 68.97 & 61.20 & 67.04 & 55.40 & 68.35 & 65.48 & 65.32 & 54.09 & 54.96 & 35.12 & 61.80 & 66.16 & 56.26 & 36.46 & \bf{73.46} \\
$\mathrm{CI}$ & $\mathrm{R}$ & 69.32 & 67.30 & 67.05 & 66.74 & 62.17 & 66.59 & 55.34 & 68.08 & 74.13 & 74.67 & 67.35 & 65.81 & 48.05 & 72.98 & 62.32 & 66.06 & 53.27 & \bf{76.66} \\
$\mathrm{CP}$ & $\mathrm{R}$ & 61.39 & 67.42 & 67.49 & 68.27 & 57.85 & 66.98 & 58.02 & 68.72 & 65.56 & 62.86 & 61.20 & 57.47 & 39.49 & 66.32 & 63.80 & 60.79 & 50.57 & \bf{72.87} \\
$\mathrm{IP}$ & $\mathrm{R}$ & 67.76 & 69.24 & 67.10 & 68.71 & 65.56 & 66.77 & 60.70 & 68.50 & 72.66 & 73.98 & 69.99 & 64.00 & 56.57 & 69.48 & 58.84 & 64.53 & 55.53 & \bf{74.16} \\
\hline
$\mathrm{SR}$ & $\mathrm{C}$ & 70.77 & 75.86 & 78.80 & 70.44 & 70.61 & 81.56 & 80.29 & 79.71 & 64.41 & 73.71 & 70.04 & 74.94 & 54.57 & 72.57 & 78.61 & 72.74 & 68.29 & \bf{81.00} \\
$\mathrm{SI}$ & $\mathrm{C}$ & 69.55 & 74.43 & 78.51 & 70.33 & 70.33 & 80.24 & 80.23 & 80.91 & 72.00 & 70.53 & 70.29 & 74.94 & 57.14 & 78.00 & 80.57 & 72.00 & 64.00 & \bf{85.25} \\
$\mathrm{SP}$ & $\mathrm{C}$ & 63.68 & 72.37 & 78.69 & 70.84 & 70.79 & 80.19 & 79.31 & \bf{80.29} & 71.51 & 70.77 & 70.04 & 66.37 & 53.43 & 66.25 & 78.37 & 65.14 & 66.29 & 75.75 \\
$\mathrm{RI}$ & $\mathrm{C}$ & 76.41 & 80.83 & 79.79 & 79.31 & 79.49 & 82.41 & 83.16 & 82.80 & 79.11 & 82.29 & 81.80 & 80.32 & 82.00 & 80.50 & 81.31 & 77.14 & 80.29 & \bf{87.83} \\
$\mathrm{RP}$ & $\mathrm{C}$ & 71.75 & 79.29 & 79.91 & 79.80 & 79.87 & 82.37 & 82.47 & 82.17 & 74.20 & 80.82 & 80.82 & 76.17 & 67.14 & 75.75 & 81.31 & 74.69 & 73.96 &\bf{83.67} \\
$\mathrm{IP}$ & $\mathrm{C}$ & 78.12 & 77.86 & 79.67 & 79.71 & 79.66 & 81.06 & 82.41 & 83.23 & 81.06 & 82.77 & 81.80 & 80.57 & 81.14 & 82.25 & 81.06 & 76.17 & 84.86 & \bf{86.33} \\
\hline
$\mathrm{SR}$ & $\mathrm{I}$ & 67.59 & 79.09 & 79.14 & 76.33 & 77.47 & 79.46 & 76.44 & 79.17 & 73.47 & 77.14 & 77.43 & 72.49 & 52.29 & 76.50 & 73.47 & 70.04 & 60.86 & \bf{81.08} \\
$\mathrm{SC}$ & $\mathrm{I}$ & 72.98 & 76.11 & 76.97 & 75.31 & 74.94 & 77.33 & 68.76 & 78.83 & 75.18 & 76.57 & 74.57 & 75.43 & 54.00 & 77.50 & 75.43 & 73.71 & 67.14 & \bf{79.83} \\
$\mathrm{SP}$ & $\mathrm{I}$ & 61.71 & 78.80 & 77.44 & 75.87 & 76.79 & 78.20 & 76.96 & \bf{80.60} & 69.80 & 63.43 & 61.96 & 66.61 & 47.71 & 73.00 & 74.94 & 63.68 & 62.86 & 71.42 \\
$\mathrm{RC}$ & $\mathrm{I}$ & 75.92 & 77.20 & 78.11 & 75.94 & 76.31 & 78.99 & 74.30 & 79.94 & 76.41 & 78.61 & 77.14 & 76.65 & 72.00 & 78.75 & 74.45 & 72.25 & 77.43 & \bf{85.67} \\
$\mathrm{RP}$ & $\mathrm{I}$ & 72.98 & 79.69 & 78.59 & 76.46 & 77.99 & 79.76 & 81.11 & 81.63 & 69.06 & 76.41 & 75.92 & 71.02 & 58.86 & 72.50 & 74.94 & 73.23 & 78.00 & \bf{84.08} \\
$\mathrm{CP}$ & $\mathrm{I}$ & 63.92 & 76.91 & 76.23 & 75.49 & 75.63 & 77.73 & 74.81 & 81.37 & 69.06 & 72.25 & 72.00 & 67.35 & 54.00 & 69.25 & 73.71 & 65.63 & 74.29 & \bf{81.42} \\
\hline
$\mathrm{SR}$ & $\mathrm{P}$ & 61.23 & 70.59 & 70.34 & 69.16 & 68.20 & 70.70 & 67.53 & 70.76 & 60.25 & 65.71 & 67.43 & 61.71 & 41.43 & 64.00 & 65.39 & 61.23 & 57.71 & \bf{71.33} \\
$\mathrm{SC}$ & $\mathrm{P}$ & 61.96 & 68.70 & 68.33 & 65.14 & 63.66 & 68.13 & 58.31 & 68.81 & 58.53 & 66.00 & 64.57 & 57.55 & 35.71 & 65.25 & 64.41 & 60.74 & 45.71 & \bf{68.83} \\
$\mathrm{SI}$ & $\mathrm{P}$ & 61.47 & 68.36 & 68.17 & 67.04 & 65.91 & 68.30 & 64.33 & 68.76 & 63.18 & 67.43 & 67.71 & 57.80 & 39.71 & 65.00 & 64.17 & 60.25 & 54.29 & \bf{68.83} \\
$\mathrm{RC}$ & $\mathrm{P}$ & 64.17 & 71.41 & 72.79 & 68.00 & 66.06 & 71.47 & 66.54 & 72.04 & 66.86 & 66.86 & 66.86 & 66.37 & 57.43 & 67.00 & 63.43 & 60.49 & 71.71 & \bf{72.05} \\
$\mathrm{RI}$ & $\mathrm{P}$ & 66.12 & 70.77 & 72.14 & 69.71 & 68.29 & 71.60 & 71.30 & 72.00 & 66.37 & 66.86 & 68.08 & 66.37 & 64.00 & 68.25 & 65.63 & 61.96 & 67.43 & \bf{73.33} \\
$\mathrm{CI}$ & $\mathrm{P}$ & 66.12 & 69.46 & 69.46 & 65.89 & 63.77 & 69.07 & 63.34 & 70.06 & 67.11 & 67.83 & 68.57 & 66.86 & 55.14 & 69.50 & 66.86 & 65.14 & 68.57 & \bf{71.42} \\
\hline
\multicolumn {2}{|c|}{Avg} & 66.99 & 71.39 & 72.29 & 71.00 & 69.14 & 72.87 & 67.56 & 73.46 & 70.59 & 70.87 & 67.85 & 68.74 & 54.82 & 71.80 & 69.78 & 65.50 & 61.07 & \bf{76.72} \\
\hline
\end{tabular}}
\label{VisDA2}
\end{table*}

\begin{table*}[h]
\centering \small
\caption{Comparison with the state-of-the-art DA methods of universal multi-source domain adaptation tasks on \textbf{VisDA2017+ImageCLEF-DA} ($M=3$, $\xi_1=\xi_2=\xi_3=0.30$ and $\xi_{12}=\xi_{23}=0.20$) dataset (Backbone: ResNet-50) measured by classification accuracy (\%). (\textbf{Com.}: source-combined. \textbf{Sin.}: single-best.)}
\setlength{\tabcolsep}{0.6mm}{
\begin{tabular}{|c|c|c|c|c|c|c|c|c|c|c|c|c|c|c|c|c|c|c|c|}
\hline
\multirow {2}{*}{$\mathcal{D}_{s_i}$} & \multirow {2}{*}{$\mathcal{D}_{t}$} & \multicolumn {2}{|c|}{Source-only} & \multicolumn {6}{|c|}{Single-best UDA} & \multicolumn {6}{|c|}{Source-combined UDA} & \multicolumn {4}{|c|}{UMDA} \\
\cline{3-20}&& Com. & Sin. & DANN & RTN & IWAN & PADA & OSBP & UAN & DANN & RTN & IWAN & PADA & OSBP & UAN & DCTN & MDAN & MDDA & \bf{UMAN} \\
\hline
$\mathrm{RCI}$ & $\mathrm{S}$ & 58.71 & 65.37 & 67.43 & 67.96 & 69.06 & 68.95 & 63.42 & 67.03 & 59.14 & 66.72 & 61.28 & 61.82 & 54.50 & 68.28 & 40.08 & 52.08 & 46.59 & \bf{70.04} \\
$\mathrm{RCP}$ & $\mathrm{S}$ & 59.44 & 66.99 & 68.87 & 67.24 & 67.74 & 68.83 & 62.82 & 68.89 & 56.32 & 63.96 & 61.25 & 60.83 & 47.10 & 60.42 & 34.36 & 49.41 & 46.15 & \bf{70.82} \\
$\mathrm{RIP}$ & $\mathrm{S}$ & 69.36 & 68.57 & 70.54 & 72.04 & 69.44 & 70.69 & 65.92 & 70.29 & 67.26 & 68.41 & 68.56 & 67.20 & 52.04 & 71.06 & 42.30 & 57.31 & 49.42 & \bf{71.48} \\
$\mathrm{CIP}$ & $\mathrm{S}$ & 66.28 & 62.92 & 65.85 & 67.64 & 62.97 & 65.75 & 57.13 & 65.66 & 67.21 & 67.09 & 67.26 & 65.87 & 52.64 & \bf{78.42} & 42.91 & 60.90 & 46.62 & 67.33 \\
\hline
$\mathrm{SCI}$ & $\mathrm{R}$ & 57.81 & 66.63 & 67.20 & 67.30 & 63.31 & 66.72 & 52.70 & 68.02 & 65.85 & 60.20 & 48.59 & 51.68 & 43.94 & 68.24 & 41.97 & 56.52 & 33.59 & \bf{69.23} \\
$\mathrm{SCP}$ & $\mathrm{R}$ & 54.77 & 66.72 & 67.50 & 68.44 & 60.07 & 67.00 & 54.71 & 68.24 & 63.08 & 60.28 & 43.00 & 51.17 & 37.40 & 66.54 & 36.71 & 56.44 & 33.15 & \bf{68.60} \\
$\mathrm{SIP}$ & $\mathrm{R}$ & 64.08 & 67.93 & 67.37 & 68.59 & 64.45 & 66.93 & 56.50 & 68.32 & 67.95 & 67.68 & 56.63 & 59.74 & 45.95 & 67.79 & 37.33 & 61.53 & 32.64 & \bf{69.48} \\
$\mathrm{CIP}$ & $\mathrm{R}$ & 67.48 & 68.00 & 67.34 & 68.04 & 62.22 & 66.86 & 58.24 & 68.34 & 66.79 & 70.38 & 69.10 & 65.03 & 55.69 & 69.85 & 39.17 & 65.09 & 39.80 & \bf{70.50} \\
\hline
$\mathrm{SRI}$ & $\mathrm{C}$ & 72.74 & 77.16 & 79.06 & 73.37 & 73.50 & 81.50 & 81.23 & 81.24 & 74.45 & 69.80 & 70.29 & 77.14 & 59.14 & 76.41 & 61.71 & 75.18 & 62.29 & \bf{87.63} \\
$\mathrm{SRP}$ & $\mathrm{C}$ & 66.12 & 76.13 & 79.14 & 73.73 & 73.77 & 81.49 & 80.77 & 80.77 & 67.83 & 68.82 & 70.04 & 74.20 & 50.86 & 72.74 & 58.29 & 71.75 & 63.43 & \bf{89.00} \\
$\mathrm{SIP}$ & $\mathrm{C}$ & 69.80 & 75.06 & 78.97 & 73.67 & 73.63 & 80.50 & 80.73 & 81.53 & 74.45 & 69.06 & 80.57 & 77.39 & 57.14 & 76.17 & 61.71 & 72.74 & 56.29 & \bf{85.12} \\
$\mathrm{RIP}$ & $\mathrm{C}$ & 76.17 & 79.44 & 79.84 & 79.67 & 79.76 & 82.06 & 82.69 & 82.79 & 80.82 & 82.77 & 82.53 & 80.08 & 70.86 & 79.35 & 65.43 & 74.45 & 69.14 & \bf{86.50} \\
\hline
$\mathrm{SRC}$ & $\mathrm{I}$ & 70.53 & 77.56 & 78.17 & 75.91 & 76.36 & 78.73 & 73.63 & 79.34 & 71.26 & 76.65 & 77.39 & 72.49 & 57.71 & 71.51 & 61.71 & 73.71 & 58.00 & \bf{81.00} \\
$\mathrm{SRP}$ & $\mathrm{I}$ & 68.08 & 79.21 & 78.49 & 76.26 & 77.47 & 79.24 & 78.21 & 80.59 & 68.32 & 72.98 & 73.47 & 68.08 & 53.71 & 70.04 & 57.43 & 76.00 & 55.71 & \bf{81.38} \\
$\mathrm{SCP}$ & $\mathrm{I}$ & 70.00 & 77.97 & 77.10 & 75.70 & 76.23 & 77.94 & 74.73 & 80.41 & 66.37 & 72.00 & 70.00 & 63.68 & 50.29 & 68.32 & 59.43 & 69.43 & 48.57 & \bf{82.38} \\
$\mathrm{RCP}$ & $\mathrm{I}$ & 66.86 & 78.79 & 77.96 & 76.17 & 77.26 & 79.19 & 78.89 & \bf{81.19} & 69.55 & 72.98 & 71.02 & 71.02 & 60.29 & 70.53 & 70.57 & 75.71 & 55.43 & 80.62 \\
\hline
$\mathrm{SRC}$ & $\mathrm{P}$ & 60.74 & 70.51 & 70.37 & 67.67 & 66.17 & 70.31 & 64.81 & 70.70 & 60.98 & 65.63 & 64.89 & 61.23 & 45.71 & 60.98 & 55.71 & 66.29 & 48.29 & \bf{70.88} \\
$\mathrm{SRI}$ & $\mathrm{P}$ & 63.92 & 70.34 & 69.94 & 68.81 & 67.66 & 70.40 & 67.99 & 70.67 & 59.75 & 66.37 & 64.89 & 61.23 & 49.14 & 62.45 & 51.43 & 67.43 & 50.57 & \bf{72.25} \\
$\mathrm{SCI}$ & $\mathrm{P}$ & 58.53 & 68.93 & 68.96 & 66.47 & 65.17 & 68.51 & 62.74 & 69.21 & 60.98 & 70.29 & 60.00 & 58.53 & 43.43 & 63.18 & 47.71 & 66.00 & 49.14 & \bf{70.13} \\
$\mathrm{RCI}$ & $\mathrm{P}$ & 64.41 & 71.14 & 70.66 & 68.61 & 66.97 & 70.96 & 68.91 & \bf{71.53} & 65.14 & 66.37 & 63.43 & 64.65 & 64.86 & 66.12 & 56.00 & 68.86 & 56.86 & 71.00 \\
\hline
\multicolumn {2}{|c|}{Avg} & 65.29 & 71.77 & 72.54 & 71.16 & 69.66 & 73.13 & 68.34 & 73.74 & 66.68 & 68.92 & 66.21 & 65.65 & 52.62 & 69.42 & 51.10 & 65.84 & 50.08 & \bf{75.69} \\
\hline
\end{tabular}}
\label{VisDA3}
\end{table*}

\begin{table}[h]
\centering \small
\caption{Comparison with the state-of-the-art DA methods of universal multi-source domain adaptation tasks on \textbf{VisDA2017+ImageCLEF-DA} ($M=4$, $\xi_1=\xi_2=\xi_3=0.22$, $\xi_4=0.11$ and $\xi_{12}=\xi_{23}=\xi_{34}=0.00$) dataset (Backbone: ResNet-50) measured by classification accuracy (\%).}
\setlength{\tabcolsep}{0.5mm}{
\begin{tabular}{c|c|c|c|c|c|c|c}
\hline
\multirow {3}{*}{Standards} & Methods & \multicolumn {5}{c|}{\tabincell{c}{VisDA2017+ImageCLEF-DA\\ $(M=4)$}} & \multirow{3}{*}{Avg}\\
\cline{2-7}&Source& $\mathrm{RCIP}$ & $\mathrm{SCIP}$ & $\mathrm{SRIP}$ & $\mathrm{SRCP}$ & $\mathrm{SRCI}$ \\
\cline{2-7}&Target& $\mathrm{S}$ & $\mathrm{R}$ & $\mathrm{C}$ & $\mathrm{I}$ & $\mathrm{P}$ \\
\hline
\multirow{2}{*}{Source-only} & Combined & 55.92 & 53.40 & 76.50 & 71.25 & 60.50 & 63.51 \\
& Single-best & 63.74 & 66.66 & 76.29 & 77.49 & 69.96 & 70.83 \\
\hline
\multirow{7}{*}{\tabincell{c}{Single-best\\UDA}} 
& DANN \cite{ganin2016domain} & 65.79 & 67.21 & 79.03 & 77.70 & 69.83 & 71.91 \\
& RTN \cite{long2016unsupervised} & \bf{67.66} & 67.45 & 73.47 & 75.80 & 67.14 & 70.30 \\
& IWAN \cite{zhang2018importance} & 66.05 & 61.11 & 73.54 & 76.19 & 65.60 & 68.50 \\
& PADA \cite{cao2018partialeccv} & 67.03 & 66.75 & 81.17 & 78.41 & 69.71 & 72.61 \\
& OSBP \cite{saito2018open} & 59.32 & 53.37 & 80.99 & 73.76 & 64.01 & 66.29 \\
& UAN \cite{you2019universal} & 65.64 & 68.09 & 81.09 & 79.70 & 70.20 & 72.94 \\
\hline
\multirow{7}{*}{\tabincell{c}{Source-\\combined\\UDA}} 
& DANN \cite{ganin2016domain} & 57.95 & 59.93 & 76.25 & 64.75 & 59.75 & 63.73 \\
& RTN \cite{long2016unsupervised} & 62.61 & 57.98 & 81.50 & 67.75 & 62.75 & 66.52 \\
& IWAN \cite{zhang2018importance} & 59.81 & 41.86 & 82.25 & 69.00 & 64.75 & 63.53 \\
& PADA \cite{cao2018partialeccv} & 63.09 & 52.81 & 76.75 & 69.25 & 60.00 & 64.38 \\
& OSBP \cite{saito2018open} & 45.41 & 31.24 & 49.75 & 44.50 & 44.25 & 43.03 \\
& UAN \cite{you2019universal} & 56.11 & 56.80 & 75.00 & 67.75 & 61.00 & 63.33 \\
\hline
\multirow{4}{*}{\tabincell{c}{UMDA}} 
& DCTN \cite{xu2018deep} & 35.15 & 38.48 & 66.25 & 55.75 & 50.50 & 49.23 \\
& MDAN \cite{zhao2018adversarial} & 44.11 & 48.68 & 68.50 & 65.25 & 60.50 & 57.41 \\
& MDDA \cite{ZhaoWZGLS0HCK20} & 30.88 & 28.19 & 60.25 & 44.50 & 36.50 & 40.06 \\
& \bf{UMAN} & 62.95 & \bf{72.96} & \bf{88.00} & \bf{83.25} & \bf{70.50} & \bf{75.53} \\
\hline
\end{tabular}}
\label{VisDA4}
\end{table}

We evaluate UMAN and the above methods in various UMDA settings as shown in Table \ref{UMDAsettings}. The classification results are shown in Table \ref{Office-31}, \ref{Office-Home}, \ref{Office-Home3}, \ref{VisDA2}, \ref{VisDA3} and \ref{VisDA4}. Most non-universal methods have negative effects on domain adaptation, performing worse than no adaptation. Even universal DA method (UAN) suffers from the degradation in performance as the number of source domains increases due to the lack of the consideration of the discrepancy between source domains. The proposed UMAN achieves the best performance and outperforms compared methods by large margins in both source-combined UDA and UMDA settings. This means the proposed UMAN accurately learns useful knowledge across different source domains and promotes the adaptation between source and target domains in their common label set.

To explore the transferability criterion in UAN and UMAN, we plot the Probability Density Distribution (\textbf{PDD}) of $w^s$ and $w^t$ in Fig. \ref{Office31-W}. Compared to PDDs in Fig. \ref{A-W-UAN} and \ref{D-W-UAN}, the PDDs in the source shared and private label sets are harder to be distinguished in Fig. \ref{AD-W-UAN}. This indicates that the transferability criterion of UAN is broken by the distribution discrepancy of same classes in the combined source domain, which assumes that the samples drawn from the same class come from the same distribution. While the PDDs in the proposed UMAN have been separated ideally as shown in Fig. \ref{AD-W-UMAN}. Hence, UMAN can mitigate the distribution discrepancy not only between each source domain and target domain, but also between any two distinct source domains.

\begin{figure*}
\centering
\subfigure[UAN: A$\rightarrow$W]{
\begin{minipage}[c]{0.23\textwidth}
\centering
\includegraphics[height=1in,width=1\linewidth]{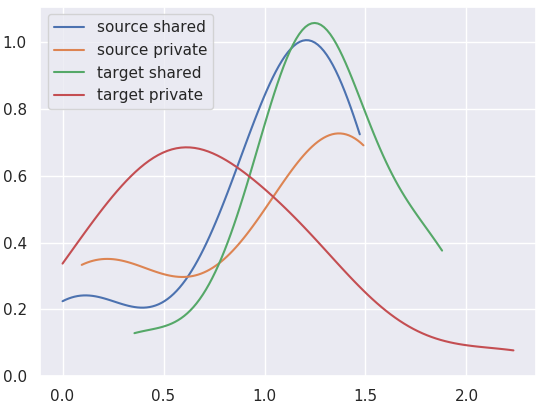}
\label{A-W-UAN}
\end{minipage}}
\subfigure[UAN: D$\rightarrow$W]{
\begin{minipage}[c]{0.23\textwidth}
\centering
\includegraphics[height=1in,width=1\linewidth]{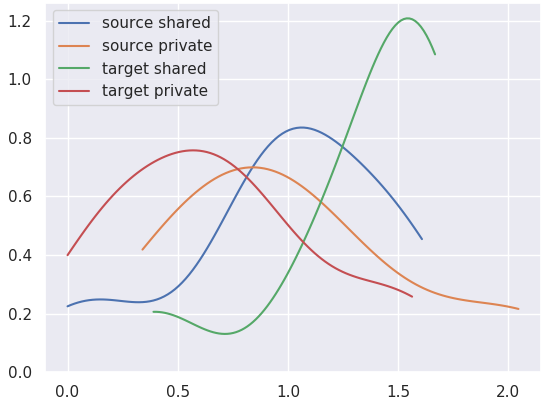}
\label{D-W-UAN}
\end{minipage}}
\subfigure[UAN: AD$\rightarrow$W]{
\begin{minipage}[c]{0.23\textwidth}
\centering
\includegraphics[height=1in,width=1\linewidth]{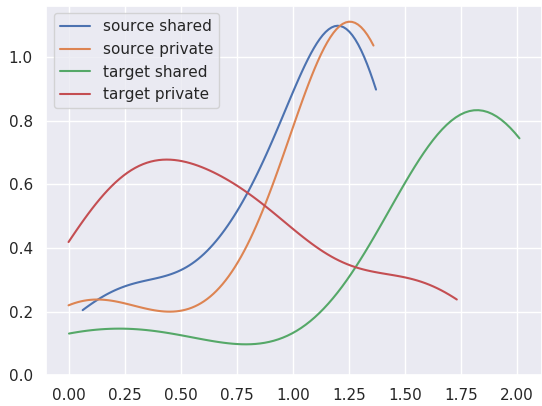}
\label{AD-W-UAN}
\end{minipage}}
\subfigure[UMAN: AD$\rightarrow$W]{
\begin{minipage}[c]{0.23\textwidth}
\centering
\includegraphics[height=1in,width=1\linewidth]{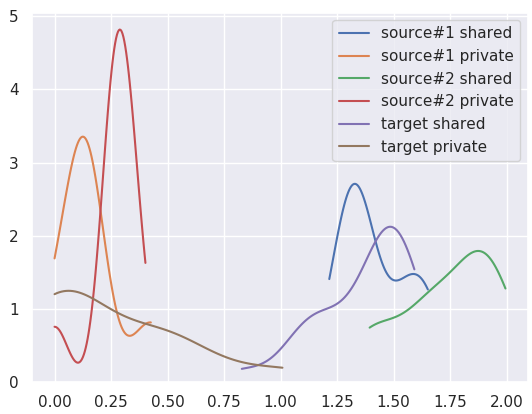}
\label{AD-W-UMAN}
\end{minipage}}
\caption{Probability Density Distribution (\textbf{PDD}) of $w^s$ on (1) 'source shared (\#i)': source samples in $\mathcal{C}$ or $\mathcal{C}_i$ in \ref{AD-W-UMAN}; (2) 'source private (\#i)': source samples in $\overline{\mathcal{C}}_{s}$ or $\overline{\mathcal{C}}_{s_i}$ in \ref{AD-W-UMAN}; and $w^t$ on (3) 'target shared': target samples in $\mathcal{C}$; (4) 'target private': target samples in $\overline{\mathcal{C}}_{t}$. Note that the source domain in \ref{AD-W-UAN} is the combination of A and D.}
\label{Office31-W}
\end{figure*}

\subsection{Analysis of UMAN Settings}
\label{UMAN Settings}
In this section, we compared the proposed UMAN with methods that achieve the best performance in universal single-source DA and closed set multi-source DA setting, i.e. UAN and MDDA respectively. We discuss UMAN settings from two aspects: (1) the number of source domains and (2) the relationship between label sets.

\begin{figure}
\centering
\subfigure[Office-31]{
\begin{minipage}[c]{0.23\textwidth}
\centering
\includegraphics[height=1in,width=1\linewidth]{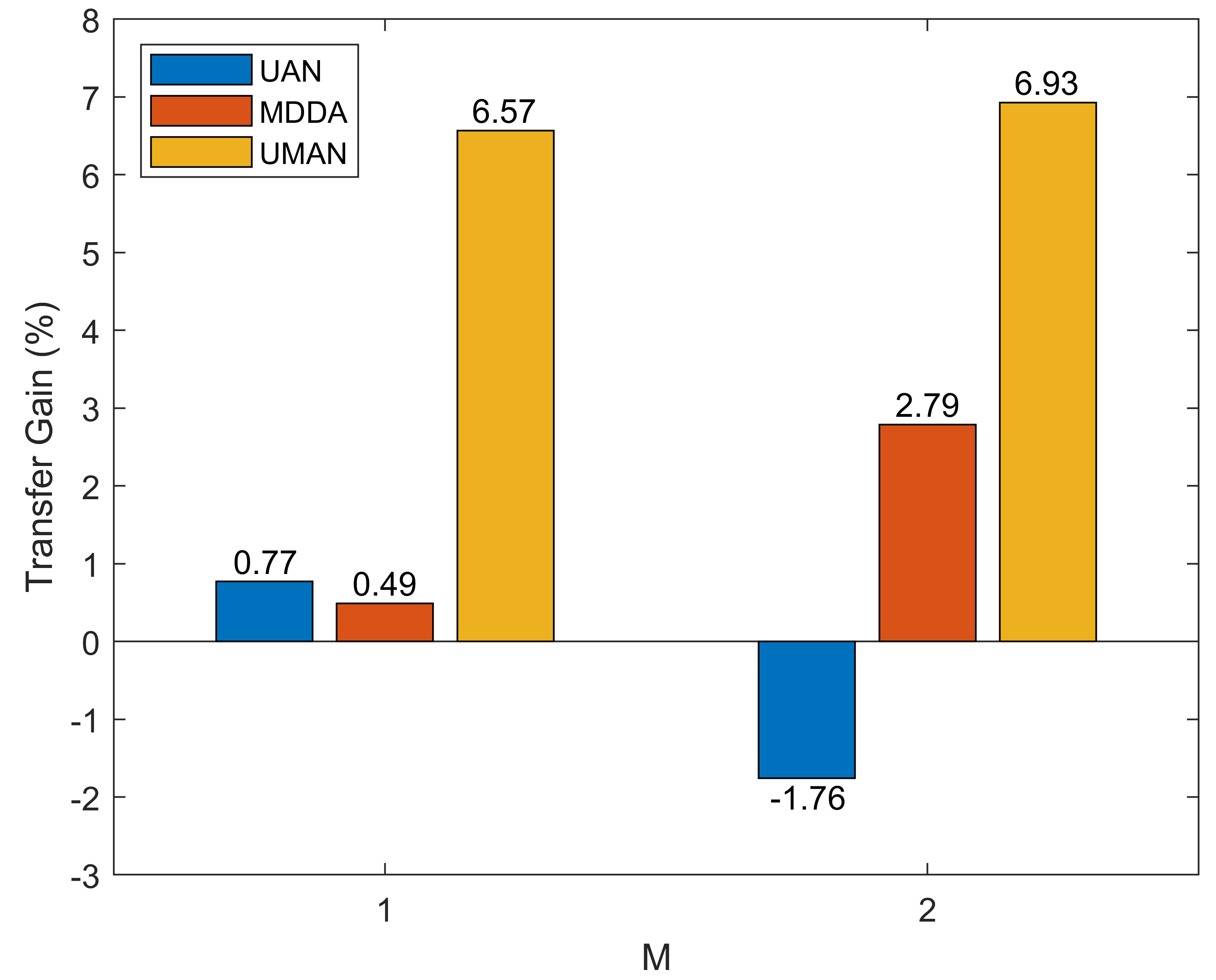}
\label{Office-31M}
\end{minipage}}
\subfigure[Office-Home]{
\begin{minipage}[c]{0.23\textwidth}
\centering
\includegraphics[height=1in,width=1\linewidth]{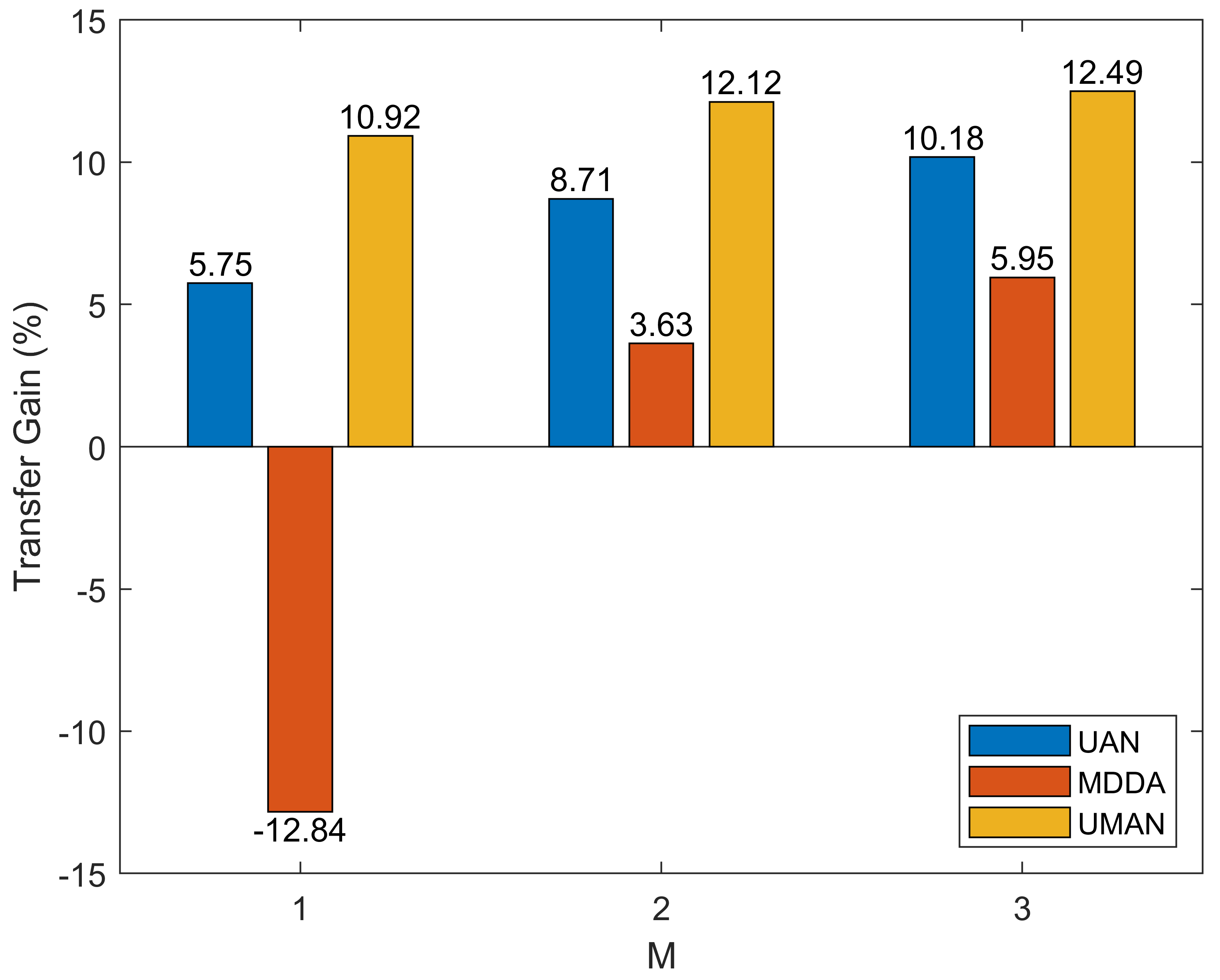}
\label{Office-HomeM}
\end{minipage}}
\subfigure[VisDA+ImageCLEF]{
\begin{minipage}[c]{0.23\textwidth}
\centering
\includegraphics[height=1in,width=1\linewidth]{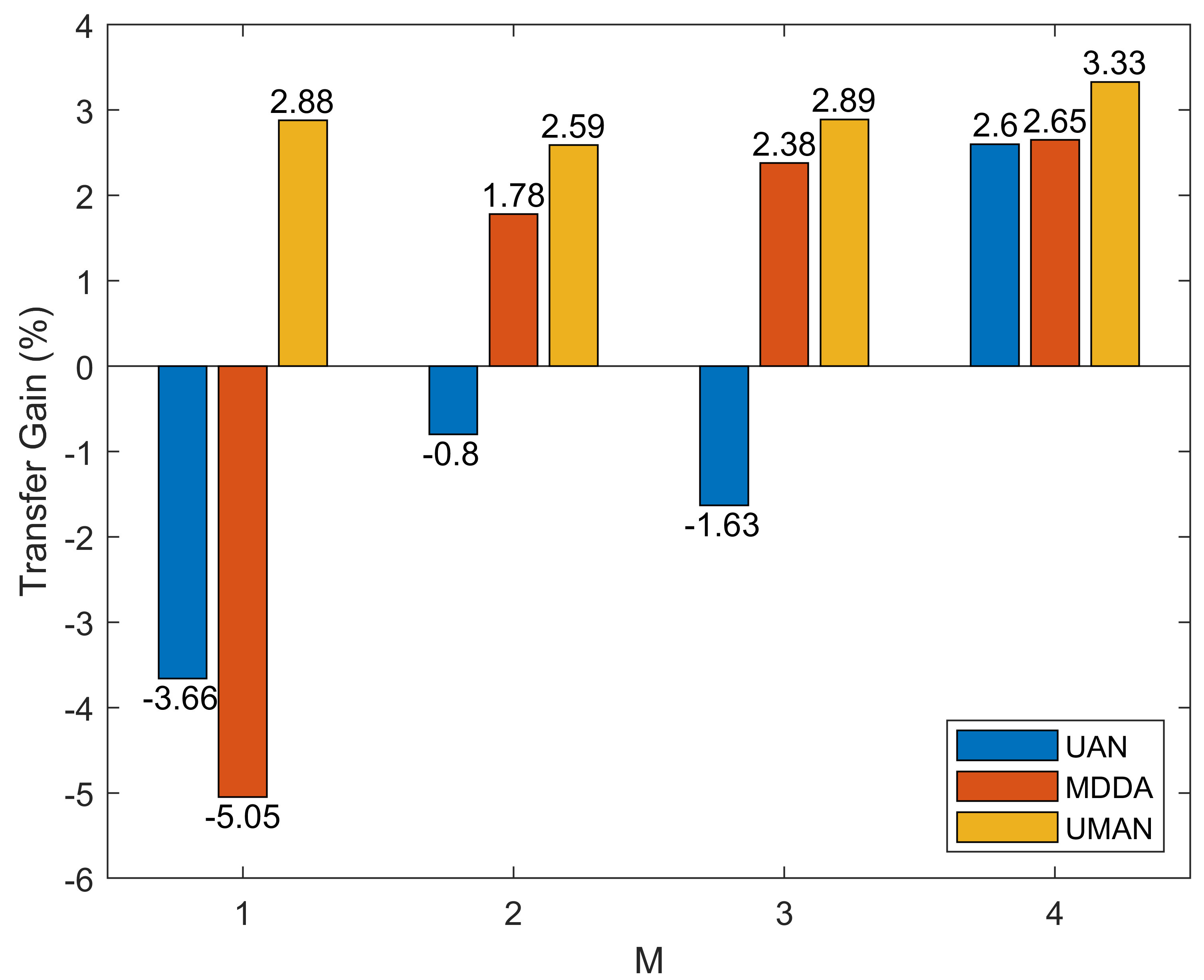}
\label{VisDA2017+ImageCLEF-DAM}
\end{minipage}}
\caption{Transfer gain of UAN (left), MDDA (middle) and UMAN (right) with respect to $M$ on different datasets. Note that negative gain means the accuracy less than that of no transfer.}
\label{vary-M}
\end{figure}

\subsubsection{The Number of Source Domains}
To analyze the impact of the number of source domains on UMDA, we fix the label set of all the source and target domains and use different number of source domains on each dataset. For \textbf{(1)} $\mathrm{D}\rightarrow\mathrm{A}$ and $\mathrm{DW}\rightarrow\mathrm{A}$ (Office-31), fix $|\mathcal{C}_i|=|\mathcal{C}|=10$, $|\overline{\mathcal{C}}_{s_i}|=|\overline{\mathcal{C}}_s|=10$, $|\overline{\mathcal{C}}_t|=11$; For \textbf{(2)} $\mathrm{Cl}\rightarrow\mathrm{Ar}$, $\mathrm{ClPr}\rightarrow\mathrm{Ar}$ and $\mathrm{ClPrRw}\rightarrow\mathrm{Ar}$ (Office-Home), fix $|\mathcal{C}_i|=|\mathcal{C}|=10$, $|\overline{\mathcal{C}}_{s_i}|=|\overline{\mathcal{C}}_s|=5$, $|\overline{\mathcal{C}}_t|=50$; For \textbf{(3)} $\mathrm{C}\rightarrow\mathrm{R}$, $\mathrm{CI}\rightarrow\mathrm{R}$, $\mathrm{CIP}\rightarrow\mathrm{R}$ and $\mathrm{SCIP}\rightarrow\mathrm{R}$ (VisDA2017+ImageClEF-DA), fix $|\mathcal{C}_i|=|\mathcal{C}|=7$, $|\overline{\mathcal{C}}_{s_i}|=|\overline{\mathcal{C}}_s|=3$, $|\overline{\mathcal{C}}_t|=2$. We show the transfer gains, which are compared to no adaptation in the source-combined setting, on three best methods in Fig. \ref{vary-M}. Generally, UAN (the best universal single-source DA method) suffers from negative transfer in some universal multi-source DA settings, and MDDA (the best closed set multi-source DA method) suffers from negative transfer in most universal single-source DA settings. The proposed UMAN outperforms the best methods in their specific DA scenarios and keeps a high transfer gain with all the UMDA settings. Meanwhile, we observe that the transfer gain of the proposed UMAN increases as the number of source domains increases. This indicates that learning across more source domains can encourage the transfer of more knowledge to the target domain.

\subsubsection{The Relationship between Label Sets}
To analyze the impact of the relationship between label sets on UMDA, we fix the number of source domains $M=2$ and select $\mathrm{DW}\rightarrow\mathrm{A}$ task in Office-31, where all the 31 classes are used.

\begin{figure*}
\centering
\subfigure[UAN]{
\begin{minipage}[c]{0.32\textwidth}
\centering
\includegraphics[height=1in,width=1\linewidth]{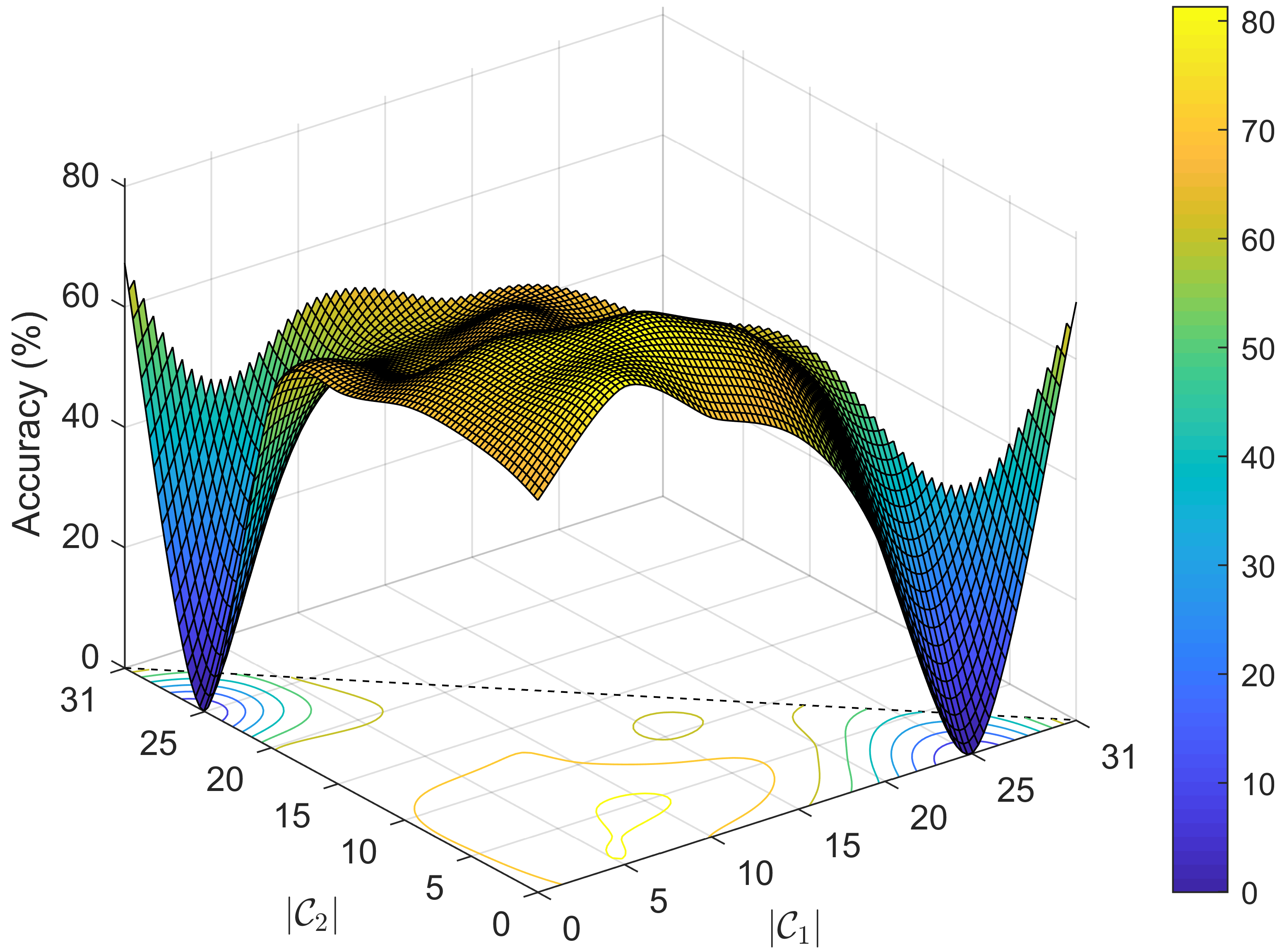}
\label{C-UAN}
\end{minipage}}
\subfigure[MDDA]{
\begin{minipage}[c]{0.32\textwidth}
\centering
\includegraphics[height=1in,width=1\linewidth]{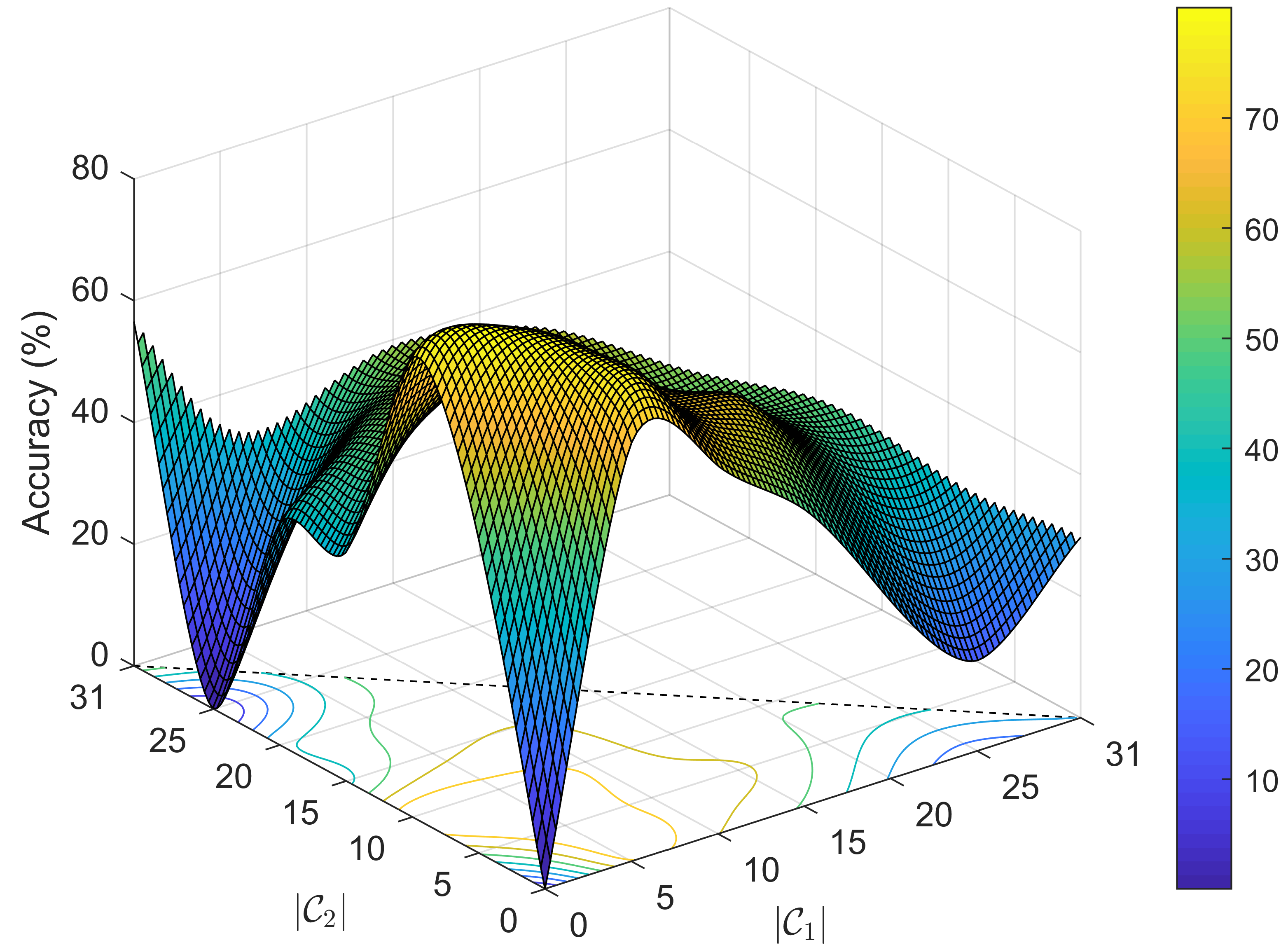}
\label{C-MDDA}
\end{minipage}}
\subfigure[UMAN]{
\begin{minipage}[c]{0.32\textwidth}
\centering
\includegraphics[height=1in,width=1\linewidth]{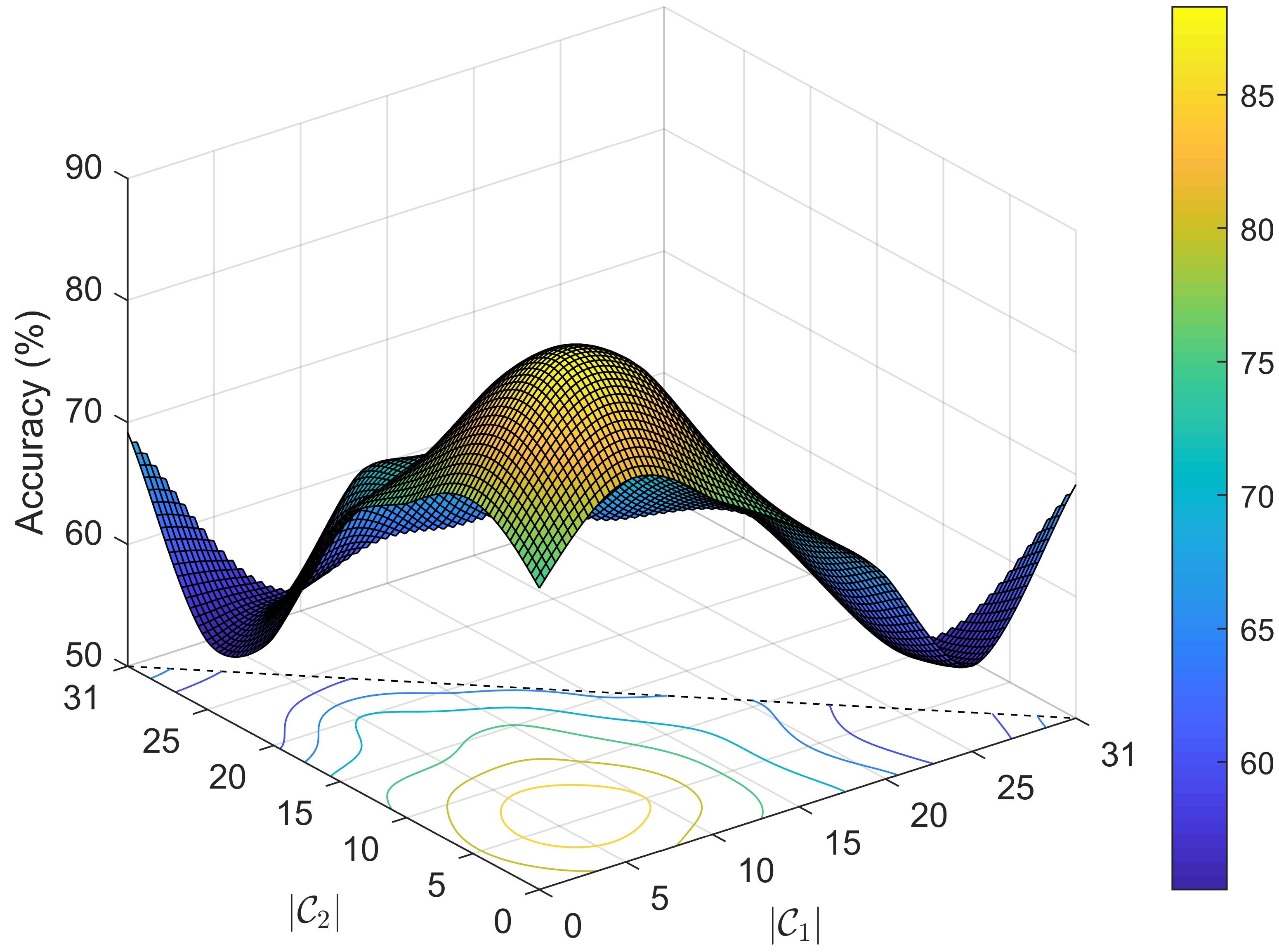}
\label{C-UMAN}
\end{minipage}}
\caption{Accuracy surface with different size of $\mathcal{C}_1$ and $\mathcal{C}_2$ on $\mathrm{DW}\rightarrow\mathrm{A}$ task ($\xi_{12}=0$).}
\label{vary-C}
\end{figure*}

\textbf{Impact of the size of the common label set}. Let $|\mathcal{C}_{s_1} \cap \mathcal{C}_{s_2}|=0$, $|\overline{\mathcal{C}}_{s_1}|=|\overline{\mathcal{C}}_{s_2}|=|\overline{\mathcal{C}}_{t}|-(|\overline{\mathcal{C}}_{t}|\mod 3)$, vary $|\mathcal{C}_{1}|$ and $|\mathcal{C}_{2}|$ from 0 to 31. Note that $|\mathcal{C}_{1}|+|\mathcal{C}_{2}|\leq31$. As shown in Fig. \ref{C-UAN}, the performance of UAN collapse at $(\mathcal{C}_1,\mathcal{C}_2)=(0,25),(25,0)$ with its default hyperparameters. In other combinations of $\mathcal{C}_1$ and $\mathcal{C}_2$, the accuracy of UAN is stable at about 70\% and the mean accuracy of UAN is 60.46\%. As shown in Fig. \ref{C-MDDA}, the performance of MDDA collapse at $(\mathcal{C}_1,\mathcal{C}_2)=(0,0),(0,25)$ with its default hyperparameters. The accuracy of MDDA varies widely with the change of $|\mathcal{C}_{1}|$ and $|\mathcal{C}_{2}|$ and the mean accuracy of MDDA is 49.40\%. In Fig. \ref{C-UMAN}, UMAN shows the more stable performance compared with UAN and MDDA, the mean accuracy of UMAN is 70.83\% and the accuracy vary from 80\% to 90\%. In conclusion, the size of the common label set has great influence on UMDA, especially when the common label sets for distinct source domains differ greatly.

\begin{figure}
\centering
\subfigure[Accuracy w.r.t. $|\mathcal{C}_t|$]{
\begin{minipage}[c]{0.23\textwidth}
\centering
\includegraphics[height=1in,width=0.9\linewidth]{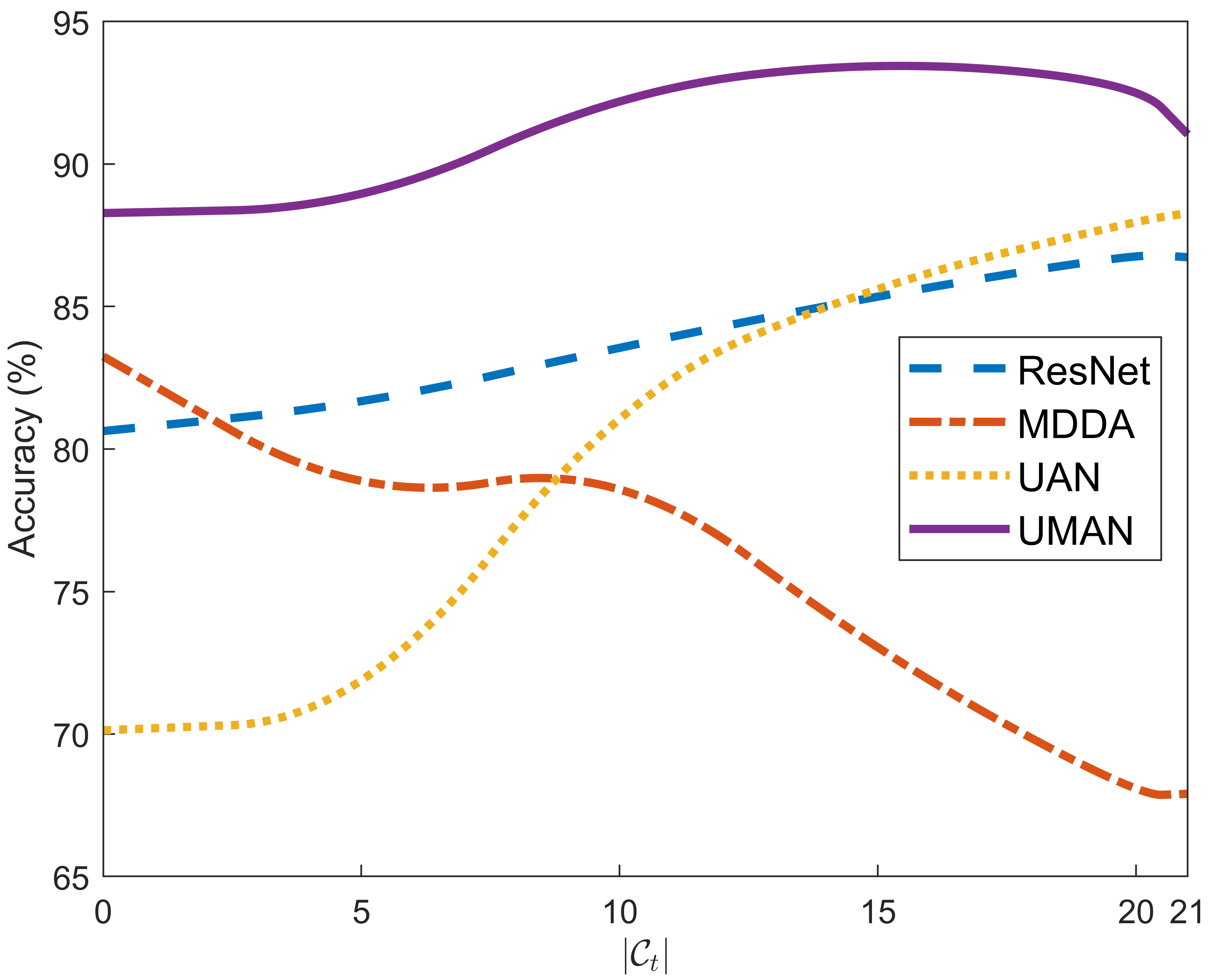}
\label{vary-Ct}
\end{minipage}}
\subfigure[Accuracy w.r.t. $|\mathcal{C}_{1} \cap \mathcal{C}_{2}|$]{
\begin{minipage}[c]{0.23\textwidth}
\centering
\includegraphics[height=1in,width=0.9\linewidth]{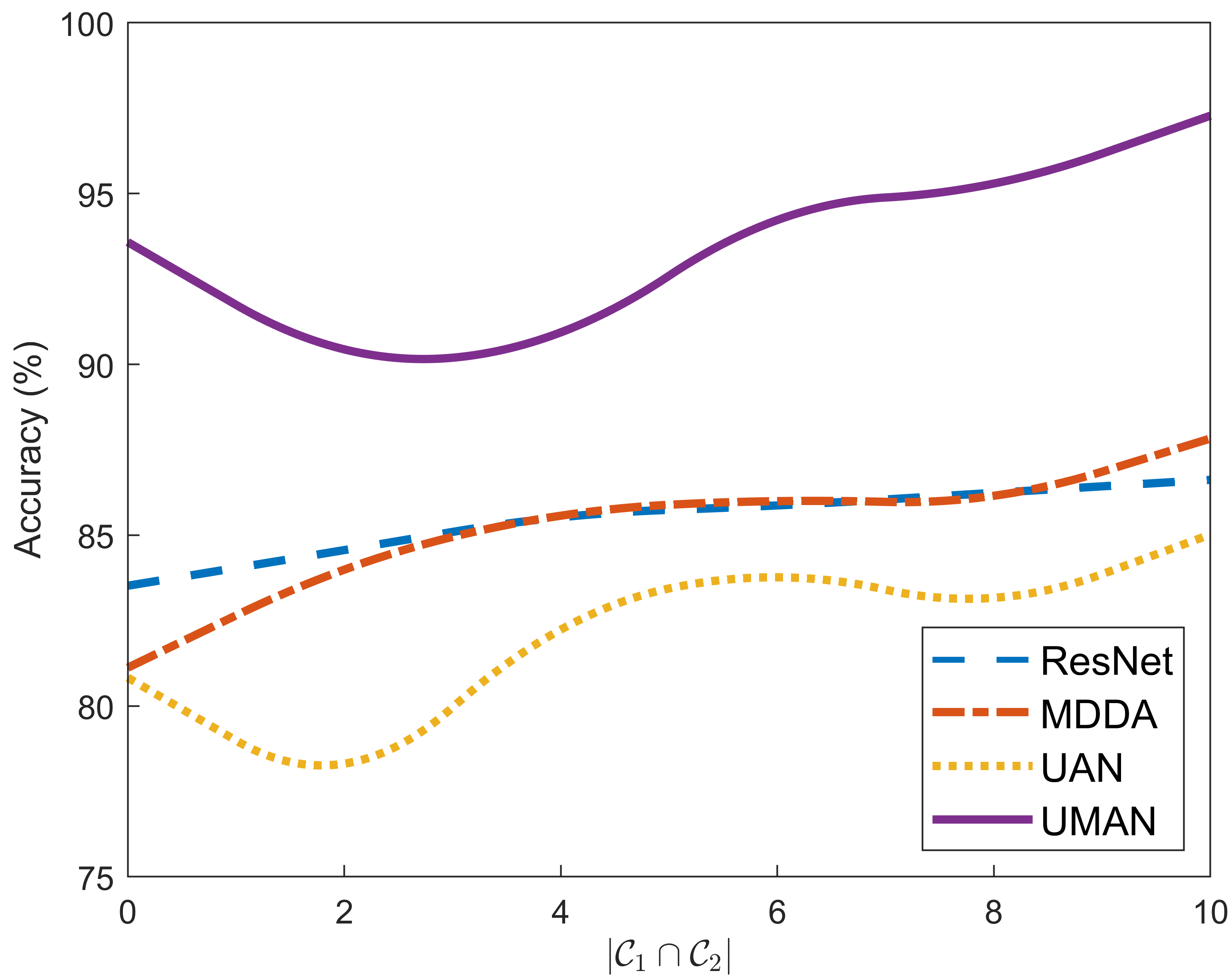}
\label{vary-C12}
\end{minipage}}
\subfigure[Accuracy w.r.t. $|\overline{\mathcal{C}}_{s_1}\cap\overline{\mathcal{C}}_{s_2}|$]{
\begin{minipage}[c]{0.23\textwidth}
\centering
\includegraphics[height=1in,width=0.9\linewidth]{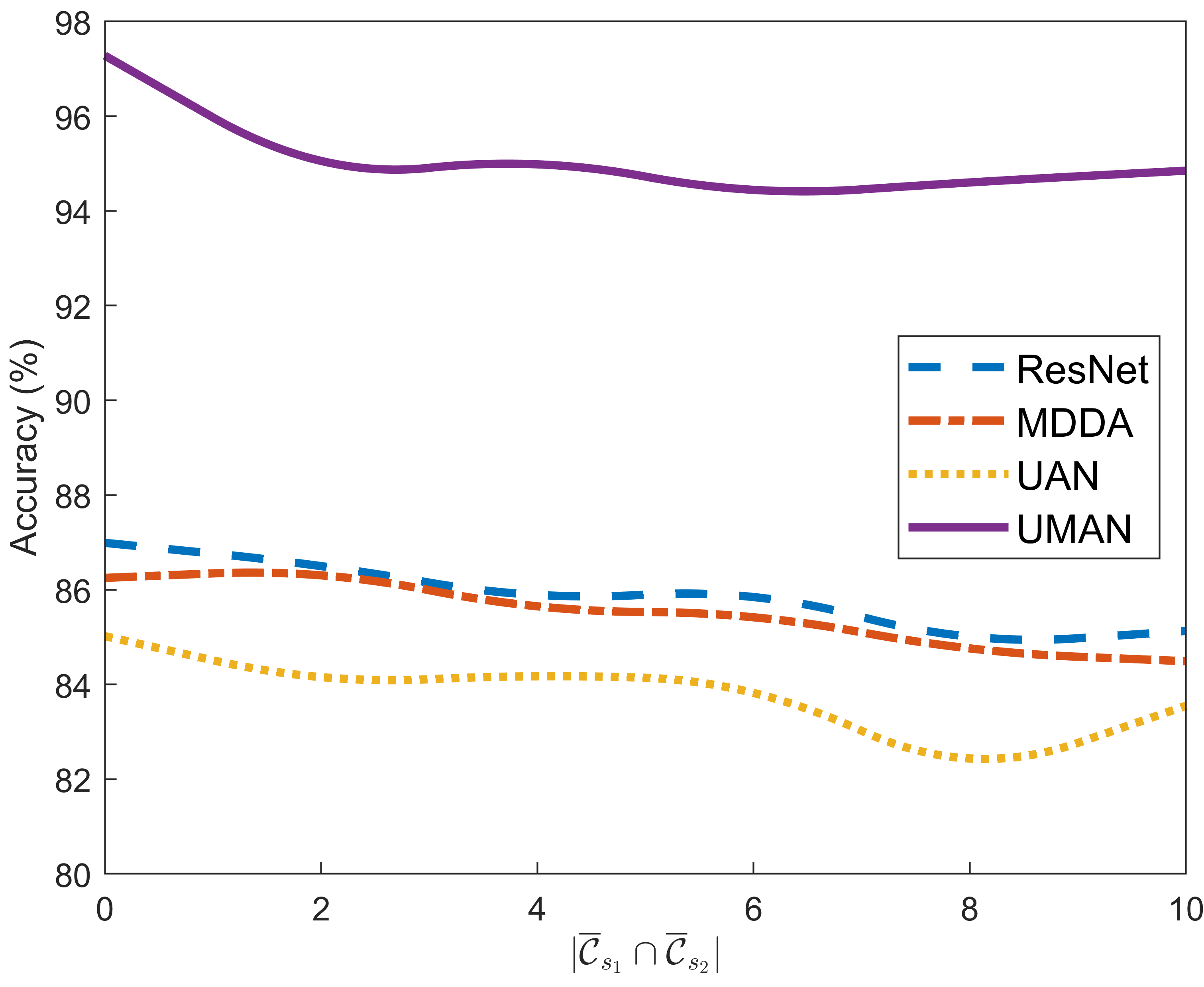}
\label{vary-Cs12}
\end{minipage}}
\caption{Accuracy w.r.t. different components of $\xi_1$, $\xi_2$, $\xi_{12}$ and the threshold on $\mathrm{DW}\rightarrow\mathrm{A}$ task.}
\label{vary}
\end{figure}

\textbf{Impact of the size of target private label set}. Let $|\mathcal{C}_{1}|=|\mathcal{C}_{2}|=5$, $|\mathcal{C}|=10$, $|\mathcal{C}_{s_1} \cap \mathcal{C}_{s_2}|=0$ and $|\overline{\mathcal{C}}_{s_1}|=|\overline{\mathcal{C}}_{s_2}|-(|\overline{\mathcal{C}}_{s_2}|\mod 2)$, vary $|\overline{\mathcal{C}}_{t}|$ from 0 to 21. As shown in Fig. \ref{vary-Ct}, UAN suffers from negative transfer when $|\overline{\mathcal{C}}_{t}|<15$ and MDDA suffers from negative transfer when $|\overline{\mathcal{C}}_{t}|<2$. The proposed UMAN improve the adaptation with all size of $\overline{\mathcal{C}}_{t}$ in UMDA. This indicates that specific universal single-source DA methods and closed set multi-source DA methods are sensitive to the size of target private label set, while UMAN not. 

\textbf{Impact of the size of the intersection of the source common label sets}. Let $|\mathcal{C}|=10$, $|\overline{\mathcal{C}}_{s_1}|=|\overline{\mathcal{C}}_{s_2}|=5$, $|\mathcal{C}_{s}|=20$, $|\overline{\mathcal{C}}_{t}|=11$ and vary $|\mathcal{C}_{1} \cap \mathcal{C}_{2}|$ from 0 to 10. For simplicity, we let $|\mathcal{C}_{1}|=|\mathcal{C}_{2}|-(|\mathcal{C}_{2}|\mod2)$. As shown in Fig. \ref{vary-C12}, UAN is affected by distribution discrepancy in multiple source domains, and it cannot promote the performance with various sizes of the intersection. MDDA cannot promote the adaptation compared with source-only. While the proposed UMAN accomplishes domain adaptation tasks and improves the performance significantly. Overall, a larger size of the intersection of the source common label sets benefits UMDA.

\textbf{Impact of the size of the intersection of the source private label sets}. Let $|\mathcal{C}_1|=|\mathcal{C}_{2}|=|\mathcal{C}|=10$, $|\overline{\mathcal{C}}_{s}|=10$ and $|\overline{\mathcal{C}}_{t}|=11$, vary $|\overline{\mathcal{C}}_{s_1}\cap\overline{\mathcal{C}}_{s_2}|$ from 0 to 10. For simplicity, we let $|\overline{\mathcal{C}}_{s_1}|=|\overline{\mathcal{C}}_{s_2}|-(|\overline{\mathcal{C}}_{s_2}|\mod2)$. As shown in Fig. \ref{vary-Cs12}, UAN and MDDA cannot promote UMDA, while the proposed UMAN keeps a large transfer gain with varying size of $\overline{\mathcal{C}}_{s_1}\cap\overline{\mathcal{C}}_{s_2}$. In conclusion, the intersection of the source private label sets has little effect on UMDA, and no intersection is advocated.

\subsection{Theoretical Verification}
\label{Analysis of UMAN}
In Section \ref{Theoretical Guarantees for UMAN}, we provide two theorems for UMAN. Here, we implement feature dimensionality reduction and visualization to verify these two theorems. As shown in Fig. \ref{T1}, samples in the common label set of two source domains are aligned and others are separated. This is consistent with Theorem \ref{theorem1}. In Fig. \ref{T2}, samples in the common label set of source and target domains are aligned and others are clustered in their private label set. This is consistent with Theorem \ref{theorem2}.

\begin{figure}
\centering
\subfigure[Verification of Theorem \ref{theorem1}]{
\begin{minipage}[c]{0.23\textwidth}
\centering
\includegraphics[height=1in,width=1\linewidth]{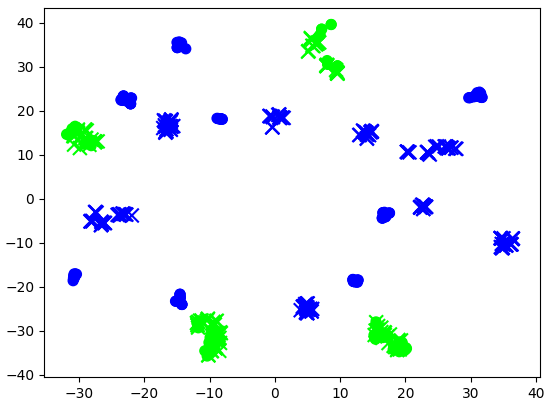}
\label{T1}
\end{minipage}}
\subfigure[Verification of Theorem \ref{theorem2}]{
\begin{minipage}[c]{0.23\textwidth}
\centering
\includegraphics[height=1in,width=1\linewidth]{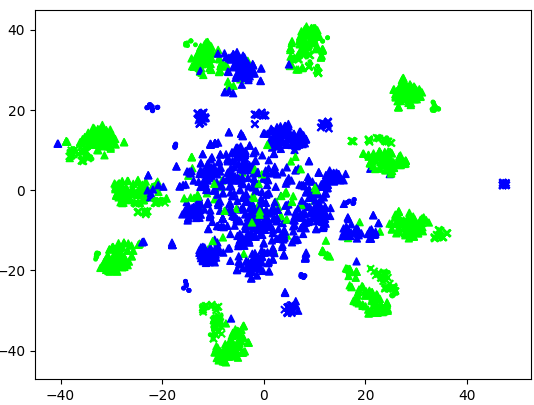}
\label{T2}
\end{minipage}}
\caption{\textbf{t-SNE visualization} on $\mathrm{DW}\rightarrow\mathrm{A}$ task. In Fig. \ref{T1}, samples come from both $\mathrm{D}$ (circle) and $\mathrm{W}$ (cross) are colored by their labels: green in their common label set ($\mathcal{C}_{s_1}\cap\mathcal{C}_{s_2}$) and blue in their private label set. In Fig. \ref{T2}, samples come from both $\mathrm{D}$ (circle), $\mathrm{W}$ (cross) and $\mathrm{A}$ (triangle) are colored by their labels: green in their common label set ($\mathcal{C}$) and blue in their private label set.}
\label{T}
\end{figure}

\section{Conclusion}
In this paper, we formulated a new domain adaptation setting called universal multi-source domain adaptation (UMDA), where the relationship between each source label set and target label set is unknown. Compared to universal domain adaptation and multi-source domain adaptation, the new challenge in UMDA can be summarized as follows. Firstly, domain discrepancy exists not only between each source domain and target domain, but also between any two distinct source domains. Secondly, as the number of source domains increases, the relationship between label sets of source and target domains becomes more complex. To address these challenges, we proposed a universal multi-source adaptation network (UMAN) with a novel margin vector for UMDA. The margin vector helps adversarial training to better align the multiple source domains and target domain in the common label set, and doesn't incur any increase of model complexity as the number of source domains increases. Moreover, the theoretic guarantees were provided for alignment of distribution of source and target domains in the common label set. Massive experimental results show that UMAN outperforms the state-of-the-art domain adaptation methods by large margins. 

\section*{Acknowledgments}
This work was supported in part by the National Natural Science Foundation of China under Grant 62071242, 61671252, 61571233 and 61901229; the Natural Science Research of Higher Education Institutions of Jiangsu Province under Grant 19KJB510008.

\bibliographystyle{IEEEtran}
\bibliography{UMAN}

\end{document}